\documentclass[twoside]{article}

\usepackage[accepted]{aistats2026}

\makeatletter
\def\@copyrightspace{}
\makeatother

\usepackage[utf8]{inputenc} % allow utf-8 input
\usepackage[T1]{fontenc}    % use 8-bit T1 fonts
\usepackage[colorlinks=true,allcolors=teal]{hyperref}       % hyperlinks
\usepackage{url}            % simple URL typesetting
\usepackage{booktabs}       % professional-quality tables
\usepackage{amsfonts}       % blackboard math symbols
\usepackage{nicefrac}       % compact symbols for 1/2, etc.
\usepackage{microtype}      % microtypography
\usepackage{lipsum}		% Can be removed after putting your text content
\usepackage{graphicx}
\usepackage{natbib}
\usepackage{doi}
\usepackage{mathtools}
\usepackage{xspace}	
\usepackage{amsthm}
\usepackage{setspace}
\usepackage{comment}
\usepackage{amsmath,amssymb}
\usepackage{algorithm}
\usepackage{algpseudocode}
\usepackage{tablefootnote}
\usepackage{array, longtable}
\usepackage{bbm}
\usepackage{xcolor}
\usepackage{epigraph}
\usepackage{enumitem}
\usepackage{adjustbox}
\usepackage{graphicx}
\usepackage{colortbl}
\usepackage{multirow}
\usepackage{array}
\usepackage{hyperref}
\usepackage{footnotehyper}
\usepackage{fnbreak}
\usepackage{enumitem}
\usepackage{threeparttable}
\usepackage{subcaption}
\usepackage{wrapfig}

\definecolor{maroon}{cmyk}{0,0.87,0.68,0.32}
\definecolor{yellow}{cmyk}{0,0,1,0}

\newtheorem{theorem}{Theorem}

\newtheorem{lemma}{Lemma}
\newtheorem{assumption}{Assumption}

\usepackage{tcolorbox}
\usepackage{pifont}
\definecolor{mydarkgreen}{RGB}{39,130,67}
\definecolor{mydarkred}{RGB}{192,47,25}
\definecolor{mypurple}{RGB}{250, 150, 250}
\newcommand{\green}{\color{mydarkgreen}}
\newcommand{\red}{\color{mydarkred}}
\definecolor{mygreen}{rgb}{0.75,1,0.75}
\newcommand{\cmark}{{\green\ding{51}}}
\newcommand{\xmark}{{\red\ding{55}}}
\algrenewcommand{\algorithmiccomment}[1]{\textcolor{gray}{\scriptsize\itshape // #1}}

\newlist{assumlist}{enumerate}{1}
\setlist[assumlist,1]{
    label=(\alph*),
    ref=\theassumption(\alph*),
    align=left,
    leftmargin=0.4cm
}

\begin{document}

\runningtitle{Variance Reduction Methods Do Not Need to Compute Full Gradients}

\runningauthor{D. Medyakov, G. Molodtsov, S. Chezhegov, A. Rebrikov, A. Beznosikov}

\twocolumn[

\aistatstitle{Variance Reduction Methods Do Not Need to Compute Full Gradients: Improved Efficiency through Shuffling}

\aistatsauthor{
  Daniil~Medyakov\textsuperscript{1,2}\And
  Gleb~Molodtsov\textsuperscript{1,2}\And
  Savelii~Chezhegov\textsuperscript{1,2}}
\aistatsauthor{
  Alexey~Rebrikov\textsuperscript{2}\And
  Aleksandr~Beznosikov\textsuperscript{1,2}
}

\aistatsaddress{
  \textsuperscript{1}
  Federated Learning Problems Laboratory \\
  \textsuperscript{2}
  Basic Research of Artificial Intelligence Laboratory (BRAIn Lab)
}
]

\begin{abstract}
  Stochastic optimization algorithms are widely used for machine learning with large-scale data. However, their convergence often suffers from non-vanishing variance. Variance Reduction (VR) methods, such as \textsc{SVRG} and \textsc{SARAH}, address this issue but introduce a bottleneck by requiring periodic full gradient computations. In this paper, we explore popular VR techniques and propose an approach that eliminates the necessity for expensive full gradient calculations. To avoid these computations and make our approach memory-efficient, we employ two key techniques: the shuffling heuristic and the concept of \textsc{SAG/SAGA} methods. For non-convex objectives, our convergence rates match those of standard shuffling methods, while under strong convexity, they demonstrate an improvement. We empirically validate the efficiency of our approach and demonstrate its scalability on large-scale machine learning tasks including image classification problem on \textsc{CIFAR-10} and \textsc{CIFAR-100} datasets.
  % In today's world, machine learning is difficult to imagine without large training datasets and models. This has led to the application of stochastic methods for training, such as \textsc{SGD}. \textsc{SGD} provides weak theoretical guarantees of convergence related to non-decreasing variance; however, modifications such as \textsc{SVRG} and \textsc{SARAH} can mitigate this issue. These methods require occasional computation of the full gradient, which can be time-consuming. In this paper, we explore popular variance reduction algorithms that eliminate the necessity for тfull gradient computations. To make our approach memory-efficient and to avoid these computations, we employ two key techniques: the shuffling heuristic and the concept of \textsc{SAG/SAGA} methods. Consequently, we enhance existing estimates for variance reduction algorithms without requiring full gradient computations. For the non-convex objective function, our estimate matches that of classic shuffling methods, while for the strongly convex one, it represents an improvement. We conduct a comprehensive theoretical analysis and provide extensive experimental results to validate the efficiency and practicality of our methods for large-scale machine learning problems.
\end{abstract}

\section{INTRODUCTION}\label{sec:introduction}

The pursuit of enhanced performance in machine learning has resulted in increasingly large training datasets. The standard approach for scalable training in this context is to formulate the task as a finite-sum minimization problem:
\vspace{-1mm}
\begin{equation}
\label{eq:finite-sum}
% \vspace{-4mm}
\textstyle{\underset{x\in\mathbb R^d}{\min}\Bigl[f(x) = \frac{1}{n}\sum\limits_{i=1}^n f_i(x)\Bigr],}
% \vspace{-3mm}
\end{equation}
where $f_i: \mathbb R^d \rightarrow \mathbb R$ and the number of functions $n$ is large. When training machine learning models, $n$ represents the size of the training set, and $f_i(x)$ denotes the loss of the model on the $i$-th data point, where $x \in \mathbb{R}^d$ is the vector of model parameters.

Stochastic methods are widely used to solve this problem as they do not calculate full gradients at each iteration. This necessity is prohibitively expensive in real-world problems due to the large $n$ values. Most well-known stochastic methods for solving the problem are \textsc{SGD} \citep{robbins1951stochastic, moulines2011non} and its various modifications \citep{ghadimi2013stochastic, ghadimi2016mini, lan2020first}. At iteration $t$, the algorithm selects a single index $i_t$ from the set $\{1, \ldots, n\}$ and performs the following step with the stepsize $\gamma$.
%\vspace{-0.1cm}
\begin{equation*}
    x^{t+1} = x^t - \gamma\nabla f_{i_t} (x^t).
\end{equation*}
%\vspace{-0.1cm}

\paragraph{Variance reduction technique.} Despite the simplicity of the classic \textsc{SGD} method and the extensive study of its properties, it suffers from one significant drawback: the variance of its stochastic gradient estimators remains high throughout the learning process. Consequently, \textsc{SGD} with a constant learning rate achieves linear convergence only to a neighborhood of the optimal solution, whose size is proportional to the stepsize and the variance \citep{gower2020variance}. The variance reduction (VR) technique \citep{johnson2013accelerating} was proposed to address this issue. Some of the most popular methods based on this technique are \textsc{SVRG} \citep{johnson2013accelerating}, \textsc{SARAH} \citep{nguyen2017sarah, hu2019efficient}, \textsc{SAG} \citep{roux2012stochastic}, \textsc{SAGA} \citep{defazio2014saga}, \textsc{Finito} \citep{defazio2014finito}, \textsc{SPIDER} \citep{fang2018spider}. In this paper, we analyze two of these algorithms: \textsc{SVRG} and \textsc{SARAH}. 

We first outline the \textsc{SVRG} method, formalized as
%\vspace{-4mm}
\begin{align}
\label{eq:vr}
\begin{split}
    &v^t = \nabla f_{i_t} (x^t) - \nabla f_{i_t} (\omega^t) + \nabla f (\omega^t), \\
    &x^{t+1} = x^t - \gamma v^t,
\end{split}
\end{align}
%\vspace{-6mm}
where index $i_t$ is selected at the $t$-th iteration. Regarding the reference point $\omega^t$, it should be updated periodically, either after a fixed number of iterations (e.g., once per epoch) or probabilistically (as in loopless versions -- see \citep{kovalev2020don}). The goal of update mechanisms as \eqref{eq:vr} is to move beyond the limitations of naive gradient estimators. It employs an iterative process to construct and apply a gradient estimator with progressively diminished variance. This approach allows for the safe use of larger learning rates, thereby accelerating the training process. Now, we present another method, \textsc{SARAH}, which employs the recursive estimator update:
\begin{align}
\label{eq:sarah}
\begin{split}
    &v^t = \nabla f_{i_t} (x^t) - \nabla f_{i_t} (x^{t-1}) + v^{t-1}, \\&x^{t+1} = x^t - \gamma v^t. 
\end{split}
\end{align}
To converge to the optimal solution $x^*$, this methods requires periodical updates of $v^t$ using the full gradient \citep{nguyen2017sarah}. As in \textsc{SVRG}, this process restarts after either a fixed number of iterations or probabilistically \citep{li2020convergence}. The practical variant, \textsc{SARAH+} \citep{nguyen2017sarah}, outperforms \textsc{SVRG} by automating the full-gradient update schedule through a heuristic based on the ratio $\nicefrac{|v^t|}{|v^0|}$.

\paragraph{Shuffling heuristic.}
The choice of sample indices ($i_t$) at each iteration critically impacts the convergence and stability of stochastic optimization, yet is often overlooked. Below, we describe the heuristic used for this selection in our algorithms. In classic stochastic methods, the index $i_t$ is selected randomly and independently at each iteration. In turn, we take a more practical approach, utilizing the shuffling heuristic \citep{bottou2009curiously, mishchenko2020random, safran2020good}. We first permute the sequence of indices ${1, \ldots, n}$. Then, we select an index based on its position in the permutation during an epoch. There are several popular data shuffling methods. One of them is Random Reshuffle (RR), which involves shuffling the data before each epoch. Another approach is Shuffle Once (SO), where the data is shuffled only once at the beginning of training. There is also Cyclic Permutation, which accesses the data in a fixed, cyclic order without any randomness. In our study, we do not explore the differences between these approaches. Instead, we highlight one important common property among them: during one epoch, we calculate the stochastic gradient for each data sample exactly once. 

Let us formalize this setting. At each epoch $s$, we have a set of indices $\{\pi_s^0, \pi_s^1, \ldots, \pi_s^{n-1}\}$ that is a random permutation of the set $\{0, 1, \ldots, n-1\}$. Then, for example, the \textsc{SVRG} update \eqref{eq:vr} at the $t$-th iteration of this epoch transforms into
%\vspace{-1mm}
\begin{equation*}
    v^t = \nabla f_{\pi_s^t} (x_s^t) - \nabla f_{\pi_s^t} (\omega^t) + \nabla f (\omega^t).
\end{equation*}
%\vspace{-6mm}
Nevertheless, the analysis of shuffling methods has some specific details. The key difference between shuffling and independent choice is that shuffling methods lack one essential property: the unbiasedness of stochastic gradients derived from the i.i.d. sampling.
%\vspace{-0mm}
\begin{equation*}
\vspace{-1mm}
        \mathbb{E}_{\pi_s^t} \left[\nabla f_{\pi_s^t}(x_s^t)\right] \neq \frac{1}{n}\sum\limits_{i = 1}^n \nabla f_{i}(x_s^t) = \nabla f(x_s^t).
        \vspace{-1mm}
\end{equation*}
This restriction leads to a more complex analysis and non-standard proof techniques.

%This work was inspired by the list of following works, the SVRG method and Shuffling approach was investigated: \cite{mishchenko2020random, koloskova2023shuffle, gorbunov2020unified, li2020unified, malinovsky2023random}

\section{BRIEF LITERATURE REVIEW}
\paragraph{No full gradient methods.} \textsc{SVRG} \eqref{eq:vr} and \textsc{SARAH} \eqref{eq:sarah} are now the standard choices for solving finite-sum problems. Nevertheless, they necessitate the full gradient computation of the target function. There is a considerable interest in VR methods that avoid calculating full gradients. While \textsc{SAG} \citep{roux2012stochastic} and \textsc{SAGA} \citep{defazio2014saga} address this issue, they require additional memory usage, with a complexity of $\mathcal{O}(nd)$. There also were some attempts to eliminate the computation of the full gradient in \textsc{SARAH}. The first approach was proposed in \citep{nguyen2021inexact}. The authors introduced an inexact \textsc{SARAH} algorithm, where they replaced the computation of the full gradient by a mini-batch gradient estimate $\frac{1}{|S|}\sum_{i\in S} f_i(x), S\subset \{1, \ldots, n\}$. To converge to the solution with $\varepsilon$-accuracy, this algorithm requires a batch size $|S|\sim \mathcal{O}\left(\frac{1}{\varepsilon}\right)$ and a stepsize $\gamma\sim\mathcal{O}\left(\frac{\varepsilon}{L}\right)$. Another approach employs the recursive \textsc{SARAH} update of the full gradient estimator. In a set of works, the authors proposed a hybrid scheme without restarts:
\begin{equation*}
    v^t = \beta_t\nabla f_{i_t} (x^t) + (1-\beta_t)(\nabla f_{i_t} (x^t) - \nabla f_{i_t} (x^{t-1}) + v^{t-1}).
\end{equation*}
In the work \citep{liu2020optimal}, this scheme was used with a constant parameter $\beta_t = \beta$. The \textsc{STORM} method \citep{cutkosky2019momentum} considers $\beta_t$ decreasing to zero and \textsc{ZeroSARAH} \citep{li2021zerosarah} combines it with \textsc{SAG/SAGA}. 

%\vspace{-0.5cm}
\begin{table*}[t]
\centering
\small
\caption{Comparison of the Convergence Results.}
\label{table1}
\begin{threeparttable}
\begin{tabular}{|c|c|c|c|c|c|}
\hline
\textbf{Algorithm} & 
\begin{tabular}{@{}l@{}}
\textbf{No} \textbf{Full} \\  \textbf{Grad.?}
\end{tabular}
 & \textbf{Memory} & \textbf{Non-Convex} & \textbf{Strongly Convex} \\ \hline
SAGA \cite{park2020linear} & \cmark &  \red{\(\mathcal{O}(nd)\)} & \(\backslash\) & \(\mathcal{O}\left(n\frac{L^{\red{2}}}{\mu^{\red{2}}}\log\left(\frac{1}{\varepsilon}\right)\right)\) \\ \hline
IAG \cite{gurbuzbalaban2017convergence}& \cmark &  \red{\(\mathcal{O}(nd)\)} & \(\backslash\) & \(\mathcal{O}\left(n^{\red{2}}\frac{L^{\red{2}}}{\mu^{\red{2}}}\log\left(\frac{1}{\varepsilon}\right)\right)\) \\ \hline
PIAG \cite{vanli2016stronger} & \cmark &  \red{\(\mathcal{O}(nd)\)} & \(\backslash\) & \(\mathcal{O}\left(n\frac{L}{\mu}\log\left(\frac{1}{\varepsilon}\right)\right)\) \\ \hline
DIAG \cite{mokhtari2018surpassing} & \cmark &  \red{\(\mathcal{O}(nd)\)} & \(\backslash\) & \(\mathcal{O}\left(n\frac{L}{\mu}\log\left(\frac{1}{\varepsilon}\right)\right)\) \\ \hline
Prox-DFinito \cite{huang2021improved}& \cmark & \red{\(\mathcal{O}(nd)\)} & \(\backslash\) & \(\mathcal{O}\left(n\frac{L}{\mu}\log\left(\frac{1}{\varepsilon}\right)\right)\) \\ \hline
AVRG \cite{ying2020variance} & \cmark &  \(\mathcal{O}(d)\) & \(\backslash\) & \(\mathcal{O}\left(n\frac{L^{\red{2}}}{\mu^{\red{2}}}\log\left(\frac{1}{\varepsilon}\right)\right)\) \\ \hline
SVRG \cite{sun2019general} & \xmark &  \(\mathcal{O}(d)\) & \(\backslash\) & \(\mathcal{O}\left(n^{\red{3}}\frac{L^{\red{2}}}{\mu^{\red{2}}}\log\left(\frac{1}{\varepsilon}\right)\right)\) \\ \hline
SVRG \cite{malinovsky2023random} & \xmark &  \(\mathcal{O}(d)\) & \(\mathcal{O}\left(\frac{nL}{\varepsilon^2}\right)\) & \(\mathcal{O}\left(n\frac{L^{\red{\nicefrac{3}{2}}}}{\mu^{\red{\nicefrac{3}{2}}}}\log\left(\frac{1}{\varepsilon}\right)\right)^{\red{\text{(1)}}}\) \\ \hline
SARAH \cite{beznosikov2023random} & \cmark & \(\mathcal{O}(d)\) & \(\backslash\) & \(\mathcal{O}\left(n^{\red{2}}\frac{L}{\mu}\log\left(\frac{1}{\varepsilon}\right)\right)\) \\  \hline
\rowcolor{mygreen} SVRG (Algorithm \ref{alg2} in this paper) & \cmark & \(\mathcal{O}(d)\) & \(\mathcal{O}\left(\frac{nL}{\varepsilon^2}\right)\) & \(\mathcal{O}\left(n\frac{L}{\mu}\log\left(\frac{1}{\varepsilon}\right)\right)\) \\ \hline
\rowcolor{mygreen} SARAH (Algorithm \ref{alg2} in this paper) & \cmark & \(\mathcal{O}(d)\) & \(\mathcal{O}\left(\frac{nL}{\varepsilon^2}\right)\) & \(\mathcal{O}\left(n\frac{L}{\mu}\log\left(\frac{1}{\varepsilon}\right)\right)\) \\ \hline
\end{tabular}
% \vspace{-1mm}
\begin{tablenotes}
\small
%\hspace{-4mm}\begin{minipage}{\columnwidth}
    \item [] {\em Columns:} No Full Grad.?= whether the method computes full gradients, Memory = amount of additional memory.
    \item [] {\em Notation:} $\mu$ = constant of strong convexity, $L$ = smoothness constant, $n$ =  size of the dataset, $d$ = dimension of the problem, $\varepsilon$ = accuracy of the solution.
    \item [] {\red{(1)}}: In this work, there are also improved results that hold in the big data regime: $n \gg \mathcal{O}\left(\frac{L}{\mu}\right)$, but it is out of the scope of this work.
\end{tablenotes}    
\end{threeparttable}
%\vspace{-5mm}
\vspace{-4mm}
\end{table*}

We now examine methods that eliminate full-gradient computations and leverage shuffling strategies, as this combination aligns with the core objective of our work. Several studies have pursued this direction, including \textsc{IAG} \citep{gurbuzbalaban2017convergence}, \textsc{PIAG} \citep{vanli2016stronger}, \textsc{DIAG} \citep{mokhtari2018surpassing}, \textsc{SAGA} \citep{park2020linear, ying2020variance}, and \textsc{Prox-DFinito} \citep{huang2021improved}, which naively store stochastic gradients. 
Among them, \textsc{PIAG}, \textsc{DIAG}, \textsc{Prox-DFinito} provide the best oracle complexity for methods with the shuffling heuristic as of now \footnote{In \citep{malinovsky2023random}, the authors derived estimates that outperform those in the discussed works. However, these results were obtained under the big-data regime: $n \gg \nicefrac{L}{\mu}$, which is a specific assumption. For this reason, we do not consider them in the current work.}. Nevertheless, they still require $\mathcal{O}(nd)$ of extra memory, which is not optimal for large-scale problems. Methods \textsc{AVRG} \citep{ying2020variance} and the modification of \textsc{SARAH} \citep{beznosikov2023random} eliminate this issue. However, they deteriorate the oracle complexity. Thus, there are still no VR methods that obviate the need for computing full gradients while being optimal in terms of additional memory. We, in turn, address this gap.

\paragraph{Shuffling methods.} Given that our approach relies on a shuffling heuristic, we provide a review of existing shuffling techniques alongside a complexity comparison with our method. A key property of shuffling in this context is that it guarantees each component function's gradient is computed precisely once per epoch. It has been demonstrated that Random Reshuffle converges more rapidly than \textsc{SGD} on multiple practical tasks \citep{bottou2009curiously, recht2013parallel}.

Nevertheless, the theoretical estimates of shuffling methods remained significantly worse than those of \textsc{SGD}-like methods \citep{rakhlin2012making, drori2020complexity, nguyen2019tight}. A breakthrough was the work \citep{mishchenko2020random} which presented new proof techniques and approaches to interpret shuffling. In particular, the results for strongly convex problems coincided with those of \textsc{SGD} which utilized an independent choice of indices \citep{moulines2011non, stich2019unified}. However, in the non-convex case, the results remained inferior to those of classic \textsc{SGD} \citep{ghadimi2013stochastic}. Furthermore, a major issue was the requirement for a large number of epochs to obtain convergence estimates in the non-convex case. Since modern neural networks are trained on a relatively small number of epochs, this requirement is unnatural. The solution to this problem was presented in the work \citep{koloskova2024convergence}. The authors based their analysis on convergence over a shorter period, termed the correlation period, rather than over entire epoch. This technique helped to improve rate.

In these works, shuffling was analyzed primarily for vanilla \textsc{SGD}. However, \textsc{SGD} is suboptimal for finite-sum problems due to its variance \citep{zhang2020tight, allen2018katyusha}.
Shuffling was later extended to variational inequalities \citep{beznosikov2023smooth}, specifically, to \textsc{Extragradient} \citep{medyakov2024shuffling}, maintaining classic convergence rates.
Consequently, researchers focused on the shuffling heuristic in conjunction with variance reduction methods, achieving linear convergence for these methods.
% Most of all, the \textsc{SVRG} method was investigated.
% The first theoretical guarantees for the \textsc{Extragradient} method with the variance reduction scheme was obtained in \citep{medyakov2024shuffling}, who succeeded in obtaining a linear convergence estimate for methods with shuffling in the VI problem.
In finite-sum minimization, we pay special attention to the works \citep{malinovsky2023random, sun2019general}. However, the theoretical estimates presented in these papers remain significantly below those of methods with an independent choice of indices \citep{allen2016variance}. This paper improves the estimates for the strongly convex objective.

%\vspace{-1mm}
The convergence results from the aforementioned works are presented in Table \ref{table1}. Our results are compared with studies employing a shuffling setting that achieves linear convergence.

\section{CONTRIBUTIONS}\label{sec:contribution}
Our main contributions are summarized as follows.

$\bullet$ \textbf{Full gradient approximation.} We explore an approach whose key feature is approximating the full gradient of the objective function using a shuffling heuristic. Moreover, this method \textit{does not require} additional $\mathcal{O}(nd)$ memory. Instead, we iteratively construct the approximation during an epoch.\\
$\bullet$ \textbf{No need to compute full gradient in variance reduction algorithms.} We construct an analysis for the full gradient approximation and enable its application to two variance reduction methods:
\begin{itemize}
\vspace{-2mm}
    \item [(a)] Classic \textsc{SVRG} \citep{johnson2013accelerating},
    \item [(b)] \textsc{SARAH} in the closest form to one in \citep{beznosikov2023random}. We transform their algorithm to improve convergence rates.
\end{itemize}
$\bullet$ \textbf{Convergence results.} We obtain convergence results under various assumptions on the objective function. We consider both the non-convex case, which is of greatest interest in contemporary machine learning problems, and the strongly convex case. To the best of our knowledge, our convergence guarantees are superior to existing variance reduction methods that do not compute the full gradient. Furthermore, we enhance upper bounds of shuffling methods for the strongly convex case.\\
$\bullet$ \textbf{Experiments.} We empirically validate our methods on \textsc{CIFAR-10} and \textsc{CIFAR-100} using the ResNet-18 model, demonstrating that our methods achieve faster and more stable convergence than prior baselines.

\section{ASSUMPTIONS}\label{sec:setup}

We present a list of assumptions below.

% \vspace{-2mm}
\begin{assumption}\label{as1}
    Each function $f_i$ is $L$-smooth, i.e., it satisfies $\|\nabla f_i(x) - \nabla f_i(y)\| \leq L\|x - y\|$ for any $x, y \in \mathbb{R}^d$.
\end{assumption}
%\vspace{-2mm}
\begin{assumption}\label{as2}
    \ 
   %\vspace{-2mm}
    \begin{assumlist}
        \item \label{as2stronglyconvex}
        \text{Strong Convexity:} Each function $f_i$ is $\mu$-strongly convex, i.e., for any $x, y \in \mathbb{R}^d$, it satisfies 
        $$ f_i(y) \geqslant f_i(x) + \langle \nabla f_i(x), y - x \rangle + \frac{\mu}{2}\|y - x\|^2 .$$  
        %\vspace{-6mm}
        \item \label{as2nonconvex} 
        \text{Non-Convexity:}
        The function $f$ has a (may be not unique) finite minimum, i.e. $f^* = \inf\limits_{x \in \mathbb{R}^d} f(x) > - \infty$.  
    \end{assumlist}
\end{assumption}

\section{ALGORITHMS AND CONVERGENCE ANALYSIS}\label{sec:main}

\subsection{Full gradient approximation}\label{subsection:gradientapprox}

In this section, we introduce an approach based on the shuffling heuristic and \textsc{SAG/SAGA} ideas. It approximates the full gradient while optimizing memory usage by avoiding the storage of past gradient values.
The \textsc{SAG} algorithm is one of the early methods designed to improve the convergence speed of stochastic gradient methods by reducing the variance of gradient updates. In \textsc{SAG}, the update step can be written as
\begin{align}\label{eq:sag}
    \textstyle{x^{t+1} \!=
\!x^t \!-\! \frac{\gamma}{n} \Bigl( \nabla f_{i_t}(x^t) \!-\! \nabla f_{i_t}(\phi_{i_t}^t) \!+\!  \sum\limits_{j=1}^n \nabla f_j(\phi_j^t) \Bigr),\!}
\end{align}
where \(\phi_j^t\) represents the past iteration at which the gradients for the $j$-th function are considered. The core idea is to store old gradients for each function (essentially storing gradients $\nabla f_j(\phi_j)$ rather than points \(\phi_j\)), and update one component of this sum with a newly computed gradient at each iteration.
As \(i_t\) are sampled randomly, it is unclear when the gradient for a specific index was last computed. However, in the case of shuffling, we know that during an epoch, the gradient for each $\nabla f_j$ is computed. Thus, at the beginning of an epoch, we reliably approximate the gradient in the same way as in \textsc{SAG}:
\begin{align} \label{eq:withoutfullgrad}
    &\textstyle{v_{s+1} = \frac{1}{n} \sum\limits_{t=1}^n \nabla f_{\pi_s^t}(x_s^t),}
\end{align}
where \(\pi_s^t\) is the sequence of data points after shuffling at the beginning of epoch \(s\). This computation is performed without additional memory, using a simple moving average during the previous epoch:
\begin{align}
\label{eq:withoutfullgradmovingaverage}
\begin{split}
    \textstyle{\widetilde{v}_s^{t+1}} = \textstyle{\frac{t}{t+1} \widetilde{v}_s^{t} + \frac{1}{t+1} \nabla f_{\pi_s^t}(x_s^t); ~~ v_{s+1} = \widetilde{v}_s^n.}
\end{split}
\end{align}
It is straightforward to prove \eqref{eq:withoutfullgradmovingaverage} matches \eqref{eq:withoutfullgrad}, and we demonstrate this in the proof of Lemma \ref{ngl3}.

\subsection{SVRG without full gradients}\label{subsection:svrgalgorithm}

As previously mentioned, variance reduction methods face a significant obstacle: they require the computation of the full gradient once per epoch or necessitate additional memory of a larger size. To address this limitation, in this section, we propose a novel algorithm based on the classical \textsc{SVRG} that eliminates the need for full gradient computation and optimizes memory usage by avoiding the storage of past gradient values.

In Section \ref{subsection:gradientapprox}, we defined the technique to approximate the full gradient: \eqref{eq:withoutfullgrad}, \eqref{eq:withoutfullgradmovingaverage}. However, this approach introduces a question: how should we alter and use the gradient approximation \(v_s\) throughout an epoch? In \textsc{SAG} \eqref{eq:sag}, we added a new \(\nabla f_{i_t}(x^t)\) and removed its previous version $\nabla f_{i_t}(\phi_{i_t}^t)$, but this requires memory for all \(\nabla f_{i_t}\) (\(\phi_{i_t}^t\)). Building on the concept from \textsc{SVRG} \eqref{eq:vr}, rather than computing $\nabla f_{i_t}$ ($\phi_{i_t}^t$), our update uses the gradient at a reference point $\omega_s$. While \textsc{SVRG} uses the full gradient at this reference point as shown in \eqref{eq:vr}, we utilize an approximation provided by \eqref{eq:withoutfullgrad}:
\begin{align*}
    v_s^{t} &= \nabla f_{\pi_s^t}(x_s^t) - \nabla f_{\pi_s^t}(\omega_s) + v_s.
\end{align*}
Finally, we perform the step of the algorithm in Line \ref{alg2:line8}, where the approximation $v_s$ is calculated in the previous epoch as \eqref{eq:withoutfullgradmovingaverage} in Line \ref{alg2:line6}. We now present the formal description of \textsc{No Full Grad SVRG} (Algorithm \ref{alg2}). The outstanding issue is selecting the point $\omega_s$ (Line \ref{alg2:line11}). The approximation $v_s$, representing the full gradient at $\omega_s$, is derived from gradients evaluated at points between $x_{s-1}^1$ and $x_{s-1}^n$. A reasonable choice for $\omega_s$ appears to be the average of these points. However, during this epoch, we are continuously moving away from this point, computing \(\nabla f_{\pi_s^t}(x_s^t)\) and adding to the reduced gradient. Consequently, the average point changes over time.
By the end of the current epoch, it evolves into an average calculated from points ranging from $x_s^1$ to $x_s^n$. Thus, choosing \(\omega_s\) as the last point from the previous epoch is a logical compromise, since we estimate not only how far we can move during the past epoch but also how far we have moved during the current one (see Lemma \ref{ngl2}).
An intriguing question for future research is whether more frequent or adaptive updates of $\omega_s$ could further improve convergence rates \citep{allen2016variance}.

\begin{algorithm}[H]
\caption{\textsc{No Full Grad SVRG}}\label{alg2}
\begin{algorithmic}[1]
    \State \textbf{Input:} Initial points $x_0^0\in\mathbb{R}^d, \omega_0 = x_0^0$; Initial gradients $\widetilde{v}_0^0 = 0^d, v_0 = 0^d$
    \State \textbf{Parameter:} Stepsize $\gamma > 0$
    \For {epochs $s = 0, 1, 2, \ldots, S$}
    \State Sample a permutation $\pi^0_s, \dots, \pi^{n-1}_s$ of $\overline{0, n-1} $~~\Comment{Sampling depends on shuffling heuristic}
    \For {$t = 0, 1, 2, \ldots, n-1$}
    \State \label{alg2:line6} $\widetilde{v}_s^{t+1} = \frac{t}{t+1} \widetilde{v}_s^{t} + \frac{1}{t+1} \nabla f_{\pi_s^t}(x_s^t)$
    \State \label{alg2:line7} $v_s^{t} = \nabla f_{\pi_s^t}(x_s^t) - \nabla f_{\pi_s^t}(\omega_s) + v_s$
    \State \label{alg2:line8} $x_s^{t+1} = x_s^t - \gamma v_s^t$
    \EndFor
    \State $x_{s+1}^0 = x_s^n$
    \State \label{alg2:line11} $\omega_{s+1} = x_s^n$
    \State \label{alg2:line12} $\widetilde{v}_{s+1}^0 = 0^d$
    \State $v_{s+1} = \widetilde{v}_s^n$
    \EndFor
\end{algorithmic}
\end{algorithm}
\vspace{-4mm}

With this dynamic strategy, our approach improves the efficiency of gradient updates and optimizes memory usage, making it suitable for large-scale optimization problems. Now, we proceed to the theoretical analysis. The problem is examined in both non-convex and strongly convex settings.

\subsubsection{Non-convex setting}\label{subsection:svrgnonconvex}

For a more detailed analysis of this method, we examine the interim results. Our analysis is structured as follows. First, we dissect convergence over a single epoch. Next, we extend this recursively across all epochs. The crucial aspect here is understanding how gradients change within an epoch. To begin, we need to show that these changes depend on two critical factors. First, how well we approximate the true full gradient at the start of each epoch. Second, how far our updates deviate from this initial reference point as we progress through it. To obtain this, we present a lemma.

\begin{lemma}\label{ngl2}
Suppose Assumptions \ref{as1}, \ref{as2} hold. Then for Algorithm \ref{alg2} a valid estimate is
    \begin{align*}
    \textstyle{\left\| \nabla f(\omega_s) - \frac{1}{n}\sum\limits_{t=0}^{n-1} v_s^t\right\|^2 \hspace{-1mm}\leqslant}&\textstyle{ ~~2\|\nabla f(\omega_{s}) - v_s \|^2} \\
    & \textstyle{+ \frac{2L^2}{n}\sum\limits_{t=0}^{n-1} \|x_s^t - \omega_s\|^2.}
    \end{align*}
\end{lemma}

In contrast to classical SVRG, where only one term matters due to the exact computation of full gradients ($v_s = \nabla f(\omega_s)$), our algorithm avoids such computations. Instead, it introduces an additional term representing the errors in approximating these gradients. This error fundamentally reflects discrepancies between our approximation and the actual gradients at $\omega_s$. We note that $v_s$ averages stochastic gradients from previous epochs (as per Equation \ref{eq:withoutfullgrad}) with $\omega_s$ set as their final point. Thus, this error quantifies shifts among those points relative to further reference points. Beginning each epoch at $\omega_s$, we gauge both potential movements within upcoming epochs and the progress made during past ones. This perspective underscores a strategic balance involved when selecting $\omega_s$, aligning with discussions in Section \ref{subsection:svrgalgorithm}. We present estimates for these deviations through a subsequent lemma.
\begin{lemma}\label{ngl3}
Suppose Assumptions \ref{as1}, \ref{as2} hold. Let the stepsize $\gamma \leqslant \frac{1}{2Ln}$. Then for Algorithm \ref{alg2} a valid estimate is
\begin{align*}
        \textstyle{\left\| \nabla f(\omega_s) - \frac{1}{n}\sum\limits_{t=0}^{n-1} v_s^t\right\|^2 \leqslant~~}&\textstyle{ 8\gamma^2L^2n^2\|v_{s}\|^2}\\
        & \textstyle{+ 32\gamma^2L^2n^2\|v_{s-1}\|^2.}
\end{align*}
\end{lemma}

Now we are ready to present the final result of this section.

\begin{theorem}\label{nfgt1}
   Suppose Assumptions \ref{as1}, \ref{as2nonconvex} hold. Then Algorithm \ref{alg2} with $\gamma\leqslant\frac{1}{20L n}$ to reach $\varepsilon$-accuracy, where $\varepsilon^2 = \frac{1}{S}\sum\nolimits_{s=1}^{S} \|\nabla f(\omega_s)\|^2$, needs $\mathcal{O} \left(\nicefrac{nL}{\varepsilon^2}\right)$ iterations and oracle calls.
\end{theorem}
%\vspace{-3mm}

Detailed proofs of the results obtained are in Appendix, Section \ref{nfgsvrg_appendix}. We present the first variance reduction method that does not require the calculation of the full gradient and is optimal concerning additional memory in the non-convex setting. Our results demonstrate that the score is, by an order of magnitude, inferior compared to the classical \textsc{SVRG} using independent sampling: $\mathcal{O}(n L)$ versus $\mathcal{O}(n^{\nicefrac{2}{3}}L)$ \citep{allen2016variance}. This outcome reflects that we approximate the full gradient at the reference point, rather than considering it at the current state. Regarding \textsc{Shuffle SVRG}, we replicate the current optimal estimate \citep{malinovsky2023random}. Considering that our method does not necessitate the calculation of full gradients, it is valid and merits further investigation. Finally, the development of the no-full-grad option for non-convex problems is a significant contribution, as this area has not been extensively explored previously (Table \ref{table1}).

\subsubsection{Strongly convex setting}\label{subsection:svrgstronglyconvex}

Let us analyze this algorithm for the strongly convex case. Based on the proof of Theorem \ref{nfgt1}, we construct an analysis that employs the Polyak-Lojasiewicz condition (see Appendix \ref{sec:basicineq}).
\begin{theorem}\label{nfgt2}
Suppose Assumptions \ref{as1}, \ref{as2stronglyconvex} hold. Then Algorithm \ref{alg2} with $\gamma\leqslant\frac{1}{20Ln}$ to reach $\varepsilon$-accuracy, where $\varepsilon = f(\omega_{S+1})-f(x^*)$, needs $\mathcal{O} \left(\nicefrac{nL}{\mu}\log \nicefrac{1}{\varepsilon}\right)$ iterations and oracle calls.
\end{theorem}
Our results for the \textsc{No Full Grad SVRG} algorithm under strong convexity conditions are similar to those observed in the non-convex setting. Moreover, it significantly outperforms existing estimates of no-full-grad methods. When comparing our results to those of other shuffling methods (see Table \ref{table1}), our algorithm improves convergence rates while maintaining optimal extra memory. Thus, it contributes to the entire class of shuffling algorithms.

\subsection{SARAH without full gradients}\label{subsection:sarahalgorithm}

We have previously discussed that \textsc{SARAH} was designed as a variance reduction method that outperforms \textsc{SVRG} in practice and has numerous interesting applications \citep{nguyen2017sarah}. This section discusses how to modify the \textsc{SARAH} method to avoid restarts. We consider two approaches: taking steps based on an accurate full gradient or developing a version of \textsc{SARAH} that does not require full gradient computations.

\begin{algorithm}[ht]
\caption{\textsc{No Full Grad SARAH}}\label{alg:sarah}
\begin{algorithmic}[1]
    \State \textbf{Input:} Initial points $x_0^0\in\mathbb{R}^d$; Initial gradients $\widetilde{v}_0^0 = 0^d, v_0 = 0^d$
    \State \textbf{Parameter:} Stepsize $\gamma > 0$
    \For {epochs $s = 0, 1, 2, \ldots, S$}
    \State Sample a permutation $\pi^1_s, \dots, \pi^{n}_s$ of $\overline{1, n}$ \Comment{Sampling depends on shuffling heuristic}
    \State $v_s^0 = v_s$ \label{vs}
    \State $x_s^1 = x_s^0 - \gamma v_s^0$
    \For {$t = 1, 2, 3, \ldots, n$}
    \State \label{algsarah:line8} $\widetilde{v}_s^{t+1} = \frac{t-1}{t} \widetilde{v}_s^{t} + \frac{1}{t} \nabla f_{\pi_s^t}(x_s^t)$
    \State \label{algsarah:line9} $v_s^{t} = \frac{1}{n}\left(\nabla f_{\pi_s^t}(x_s^t) - \nabla f_{\pi_s^t}(x_s^{t-1})\right) + v_s^{t-1}$
    \State \label{algsarah:line10} $x_s^{t+1} = x_s^t - \gamma v_s^t$
    \EndFor
    \State $x_{s+1}^0 = x_s^{n+1}$
    \State \label{tvs} $\widetilde{v}_{s+1}^1 = 0$
    \State $v_{s+1} = \widetilde{v}_s^{n+1}$
    \EndFor
\end{algorithmic}
\end{algorithm}

There exists a version of \textsc{SARAH} that obviates the need for the full gradient computation \citep{beznosikov2023random}, however its upper estimate, $\mathcal{O}\left(n^2\nicefrac{L}{\mu}\log\nicefrac{1}{\varepsilon}\right)$, significantly deviates from the one obtained for \textsc{SVRG} in the previous section. Let us demonstrate what can be modified in their method to improve this estimate. The authors in the mentioned work employed the same technique to approximate the full gradient as \eqref{eq:withoutfullgrad}. The algorithm applied the standard \textsc{SARAH} update formula \eqref{eq:sarah}. To initiate a new recursive cycle at the start of each epoch, an extra step was taken with an approximated full gradient. In this way, they also used the \textsc{SAG/SAGA} idea but provided a recursive reduced gradient update to avoid storing all stochastic gradients during the epoch.

To continue the analysis, we aim to shed light on the differences between \textsc{SAG} and \textsc{SAGA} algorithms, as this is crucial for our modifications. We discussed the \textsc{SAG} update \eqref{eq:sag}. The
\textsc{SAGA} update is almost the same, except for the absence of the factor $\nicefrac{1}{n}$. Thus, maintaining the notation used for \textsc{SAG}, we can express the SAGA step as
\begin{align}\label{eq:saga}
    \!\!\textstyle{x^{t+1} \!=\! x^t \!-\! \gamma\Bigl(\! \nabla f_{i_t}(x^t) \!-\! \nabla f_{i_t}(\phi_{i_t}^t) \!+\!  \frac{1}{n}\!\sum\limits_{j=1}^n \!\nabla f_j(\phi_j^t)\! \Bigr).}
\end{align}
The key difference hides in the reduction of the variance of the \textsc{SAG} update in $n^2$ times compared to \textsc{SAGA} with the same $\phi$'s, however, the payback for such a gain is a non-zero bias in \textsc{SAG}. The choice of unbiasedness was made in \textsc{SAGA} primarily to develop a simple and tight theory for variance reduction methods and to provide theoretical estimates for proximal operators \citep{defazio2014saga}.

Now we can specify and state that the idea behind \textsc{SAGA} was applied in \citep{beznosikov2023random}. Nevertheless, attempting to increase variance for the sake of zero bias appears illogical here because the shuffling heuristic remains in use, thereby inherently introducing bias. As a result, achieving convergence requires very small step sizes, leading to significantly worse estimates. In contrast, we propose leveraging the concept of \textsc{SAG}, similar to what we did in \textsc{No Full Grad SVRG}, and modifying the \textsc{SARAH} update during each epoch, as
\begin{align*}
    \textstyle{v_s^{t+1} = \frac{1}{n}\left(\nabla f_{\pi_s^t}(x_s^t) - \nabla f_{\pi_s^t}(x_s^{t-1})\right) + v_s^{t-1}.}
\end{align*}
This approach enables the use of larger steps and improves convergence rates. We provide the formal description of the \textsc{No Full Grad SARAH} method (Algorithm \ref{alg:sarah}).
One can observe that we slightly modify the coefficients in the full gradient approximation scheme (Line \ref{algsarah:line8} of Algorithm \ref{alg:sarah}) compared to Algorithm \ref{alg2}. The discrepancy stems solely from differences in indexing. In this case, the full gradient approximation begins at iteration $t=1$ rather than $t=0$, as we incorporate an extra restart step not present in \textsc{SVRG}. Therefore, we make this adjustment to prevent the factor from becoming $\nicefrac{1}{(n+1)}$ instead of $\nicefrac{1}{n}$ in \eqref{eq:withoutfullgrad}. Now we proceed to the theoretical analysis of Algorithm \ref{alg:sarah} under both non-convex and strongly convex assumptions on the objective function.

\subsubsection{Non-convex setting}\label{subsection:sarahnonconvex}

During the proof of the convergence estimates of  \textsc{SARAH}, we proceed similarly to our approach for \textsc{SVRG}. Initially, we focus on a single epoch and demonstrate convergence within it. To achieve this, we estimate the difference between the gradient at the start of the epoch and the average of the reduced gradients used for updates throughout the epoch.
\begin{lemma}\label{l1:sarahmain}
    Suppose that Assumptions \ref{as1}, \ref{as2} hold. Then for Algorithm \ref{alg:sarah} a valid estimate is
    \begin{align*}
        \textstyle{\left\| \nabla f(x_s^0) - \frac{1}{n+1}\sum\limits_{t=0}^n v_s^t\right\|^2 \leqslant~~} & \textstyle{2\|\nabla f(x_{s}^0) - v_s \|^2 }\\
        & \textstyle{+ \frac{2L^2}{n+1}\sum\limits_{t=1}^n \|x_s^t - x_s^{t-1}\|^2.}
    \end{align*}
\end{lemma}
The first term is identical to that previously encountered in the analysis of Algorithm \ref{alg2} -- the difference between the true full gradient and its approximation. Additionally, the second term conveys a similar meaning to its counterpart in Lemma \ref{ngl2}. It represents the difference between the current and reference points during an epoch, with the reference point consistently set as the previous one. Thus, we follow a similar approach to \textsc{SVRG} and proceed to the next lemma.
\begin{lemma}\label{l2:sarahmain}
    Suppose that Assumptions \ref{as1}, \ref{as2} hold. Let the stepsize $\gamma \leqslant \frac{1}{3L}$. Then for Algorithm \ref{alg:sarah} a valid estimate is
    \begin{align*}
            \textstyle{\left\| \nabla f(x_s^0) - \frac{1}{n+1}\sum\limits_{t=0}^n v_s^t\right\|^2 \leqslant~~}&\textstyle{ 9\gamma^2L^2\|v_{s}\|^2} \\
            & \textstyle{+ 36\gamma^2L^2n^2\|v_{s-1}\|^2.}
    \end{align*}
\end{lemma}
Obtaining this crucial lemma, we can now present the final result of this section.
\begin{theorem}\label{th1:sarahmain}
   Suppose Assumptions \ref{as1}, \ref{as2nonconvex} hold. Then Algorithm \ref{alg:sarah} with $\gamma\leqslant\frac{1}{20L(n+1)}$ to reach $\varepsilon$-accuracy, where $\varepsilon^2 = \frac{1}{S}\sum\nolimits_{s=1}^{S} \|\nabla f(x_s^0)\|^2$, needs $\mathcal{O} \left(\nicefrac{nL}{\varepsilon^2}\right)$ iterations and oracle calls.
\end{theorem}

We obtain the expected result. Notably, the upper bound for the convergence of \textsc{No Full Grad SARAH} aligns with that of \textsc{No Full Grad SVRG}, as stated in Theorem \ref{nfgt1}. Our comparison with previous estimates is consistent and detailed in Section \ref{subsection:svrgnonconvex}.

\subsubsection{Strongly convex setting}\label{subsection:sarahstringlyconvex}
We extend our analysis on the strongly convex objective function using \eqref{PL} (see Appendix \ref{sec:basicineq}).
\begin{theorem}\label{th2:sarahmain}
    Suppose Assumptions \ref{as1}, \ref{as2stronglyconvex} hold. Then Algorithm \ref{alg:sarah} with $\gamma\leqslant\frac{1}{20L(n+1)}$ to reach $\varepsilon$-accuracy, where $\varepsilon = f(x_{S+1}^0)-f(x^*)$, needs $\mathcal{O} \left(\nicefrac{nL}{\mu}\log \nicefrac{1}{\varepsilon}\right)$ iterations and oracle calls.
\end{theorem}

\section{LOWER BOUNDS}\label{sec:lower_bound_main}
A natural question arises: is Algorithm \ref{alg2} optimal? We address its optimality in the non-convex case (in fact, we can also consider the strongly convex case, but the concept remains the same). A comprehensive explanation requires understanding the essence of \textit{smoothness} assumptions, which are used to study variance-reduced schemes. 

As a result, we present the lower bound for the non-convex finite-sum problem \eqref{eq:finite-sum} under Assumption \ref{as1}. Furthermore, we provide an explanation of why it is impossible to construct a lower bound that matches the result of Theorem \ref{nfgt1}.
\begin{theorem}[\textbf{Lower bound}]
\label{thm:lower}
    For any $L > 0$, there exists a problem \eqref{eq:finite-sum} which satisfies Assumption \ref{as1}, such that for any output of a first-order algorithm, number of oracle calls $N_c$ required to reach $\varepsilon$-accuracy is lower bounded as $N_c = \Omega\left(\nicefrac{L\Delta}{\varepsilon^2}\right).$
\end{theorem}
In fact, this result has already been stated (see Theorem 4.7 in \citep{zhou2019lower}). However, to enhance the clarity of the lower bound, we construct it in a different form. Furthermore, the interpretation of the obtained result (see Remark 4.8 in \citep{zhou2019lower}) is not entirely correct. For example, the comparison with the upper bound from \citep{fang2018spider} is inconsistent, as the smoothness parameters are considered different. Consequently, the problem classes \textit{do not coincide}.
\begin{theorem}[\textbf{Non-optimality}]
\label{thm: quest}
    For any $L > 0$, there is \textbf{no} problem \eqref{eq:finite-sum} which satisfies Assumption \ref{as1}, such that for any output of first-order algorithm, number of oracle calls $N_c$ required to reach $\varepsilon$-accuracy is lower bounded with $p > \frac{1}{2}: N_c = \Omega\left(\nicefrac{n^pL\Delta}{\varepsilon^2}\right).$
    \end{theorem}
The theorem shows that the best result that can potentially be obtained in terms of lower bounds is $\Omega\left(\nicefrac{\sqrt{n}L\Delta}{\varepsilon^2}\right)$. Therefore, the results for the upper bound (Theorems \ref{nfgt1} and \ref{th1:sarahmain}) are non-optimal, and the lower bound (Theorem \ref{thm:lower}) could be non-optimal in the class of problems induced by Assumption \ref{as1}.
Theorem \ref{thm: quest} signifies that despite superior performance compared to existing results (Table \ref{table1}), a gap remains between the upper and lower bounds. For details, see Appendix \ref{sec:lower_bound}.
%\vspace{-2mm}

\section{EXPERIMENTS}\label{sec:experiments}
% \textbf{ResNet-18 on CIFAR-10 classification.}
To assess the efficiency of our No Full Gradient (NFG) modifications to SVRG and SARAH, we perform experiments on image-classification benchmarks: \textsc{CIFAR-10} and \textsc{CIFAR-100} datasets \citep{krizhevsky2009learning} using the \textsc{ResNet18} model \citep{he2016deep}. 
In all experiments, the model is trained for 200 epochs with a batch size of 128. We apply weight decay $\lambda_{1}=5\times10^{-4}$ and a cosine learning-rate schedule, initializing at 0.1 and decaying to $10^{-3}$. For each algorithm, we record training and test curves for (a) cross-entropy loss and (b) accuracy. Curves are plotted against cumulative full-gradient computations, highlighting efficiency gains.
% Let 
% $f(w, x, y)$  denote the loss function, where $w \in \mathbb{R}^d$ represents the model parameters, $x \in \mathbb R^n$ is the input, and  $y \in \mathbb R$ is the corresponding label. We consider the following optimization problem:
% \begin{equation*}
%     \min \limits_{w} \frac{1}{M} \sum \limits_{i = 1}^M f(w, x_i, y_i) + \frac{\lambda_1}{2}\|w\|^2,
% \end{equation*}
% where $f$ is the cross-entropy loss function and $\lambda_1$ is a regularization parameter.
\subsection{Results on \textsc{CIFAR-10}}Figures~\ref{fig:svrg_10}, \ref{fig:sarah_10} present training and test metrics on \textsc{CIFAR-10}.

\begin{figure}[H]
\centering
% \small{\texttt{a9a}} & \hspace{2mm}\small{\texttt{w8a}} & \hspace{3mm}\small{\texttt{ijcnn1}} & \hspace{4mm}\small{\texttt{skin-nonskin}} \\
\begin{subfigure}[b]{0.48\columnwidth}
\includegraphics[width=\textwidth]{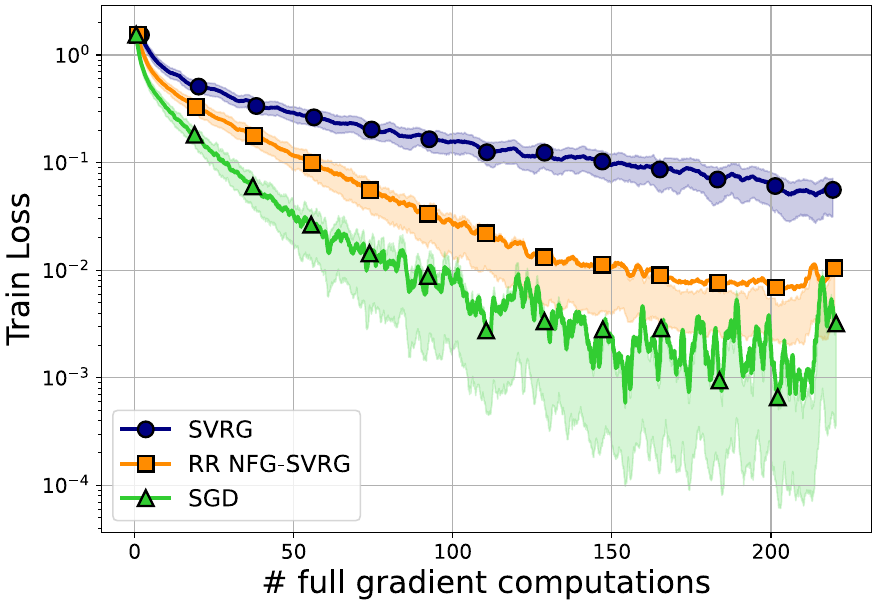}
\caption{train loss}
\end{subfigure}
\hfill
\begin{subfigure}[b]{0.48\columnwidth}
\includegraphics[width=\textwidth]{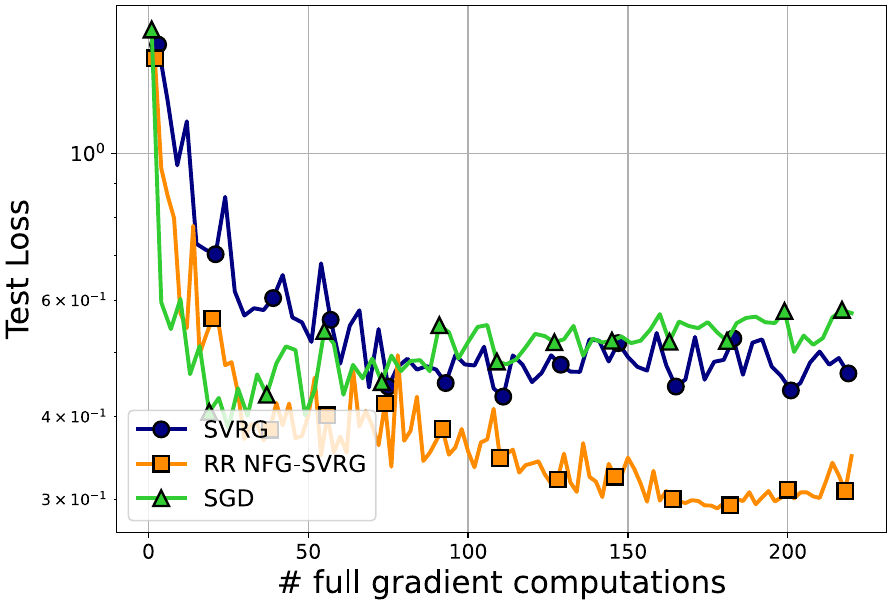}
\caption{test loss}
\end{subfigure}
\begin{subfigure}[b]{0.48\columnwidth}
\includegraphics[width=\textwidth]{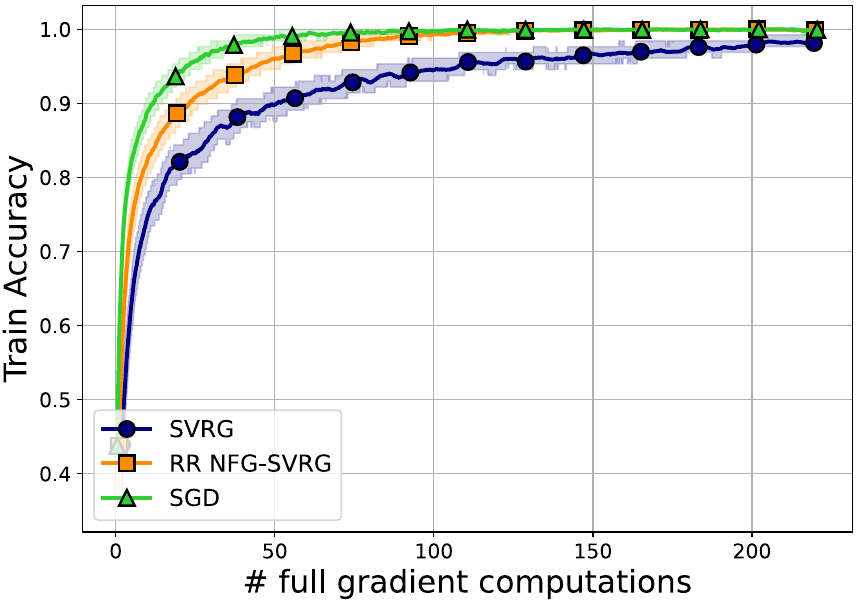}
\caption{train accuracy}
\end{subfigure}
\begin{subfigure}[b]{0.48\columnwidth}
\includegraphics[width=\textwidth]{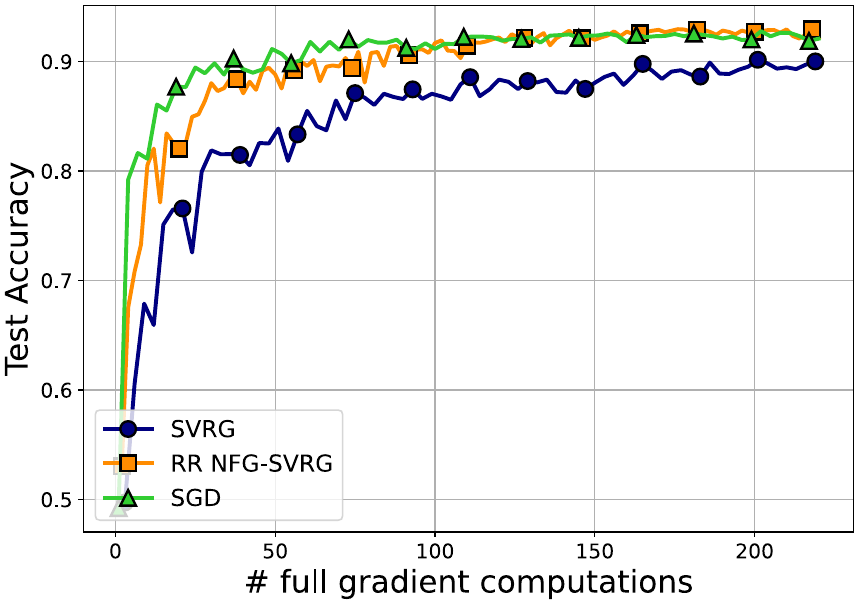}
\caption{test accuracy}
\end{subfigure}
\caption{\textsc{No Full Grad SVRG} and \textsc{SVRG}.}
\vspace{-4mm}
\label{fig:svrg_10}
\end{figure}
\textsc{NFG SVRG} reduces training loss oscillations compared to \textsc{SGD}, particularly in low-diversity datasets. Despite batch fluctuations, convergence remains smooth. On the test set, \textsc{NFG SVRG} shows better loss reduction, and test accuracy exceeds that of SGD, stabilizing from epoch 150.

\begin{figure}[H]
\centering
% \small{\texttt{a9a}} & \hspace{2mm}\small{\texttt{w8a}} & \hspace{3mm}\small{\texttt{ijcnn1}} & \hspace{4mm}\small{\texttt{skin-nonskin}} \\
\begin{subfigure}[b]{0.48\columnwidth}
\includegraphics[width=\textwidth]{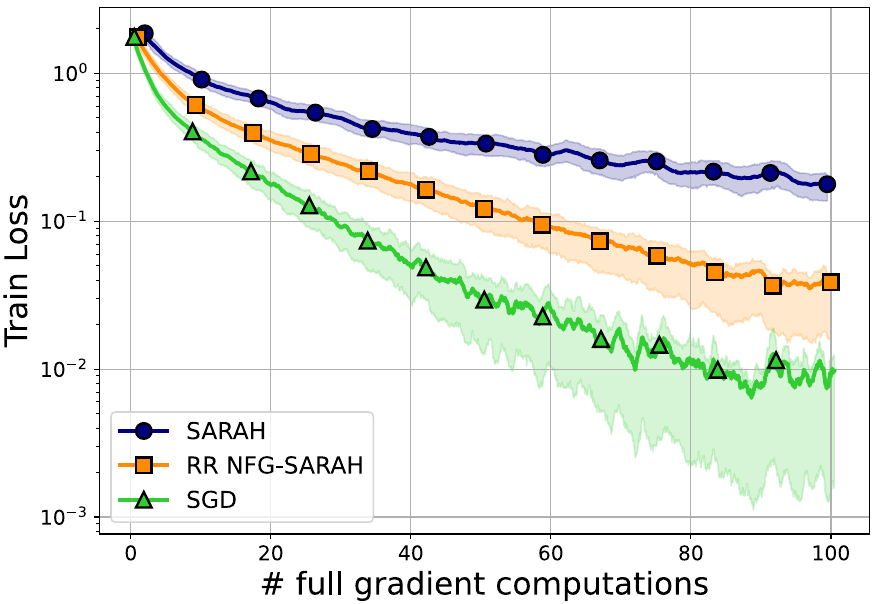}
\caption{train loss}
\end{subfigure}
\hfill
\begin{subfigure}[b]{0.48\columnwidth}
\includegraphics[width=\textwidth]{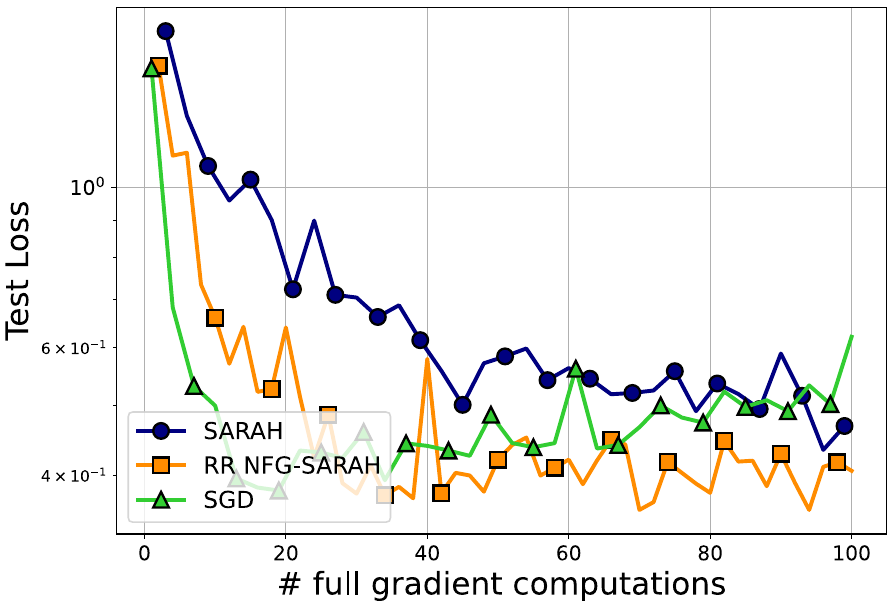}
\caption{test loss}
\end{subfigure}
\begin{subfigure}[b]{0.48\columnwidth}
\includegraphics[width=\textwidth]{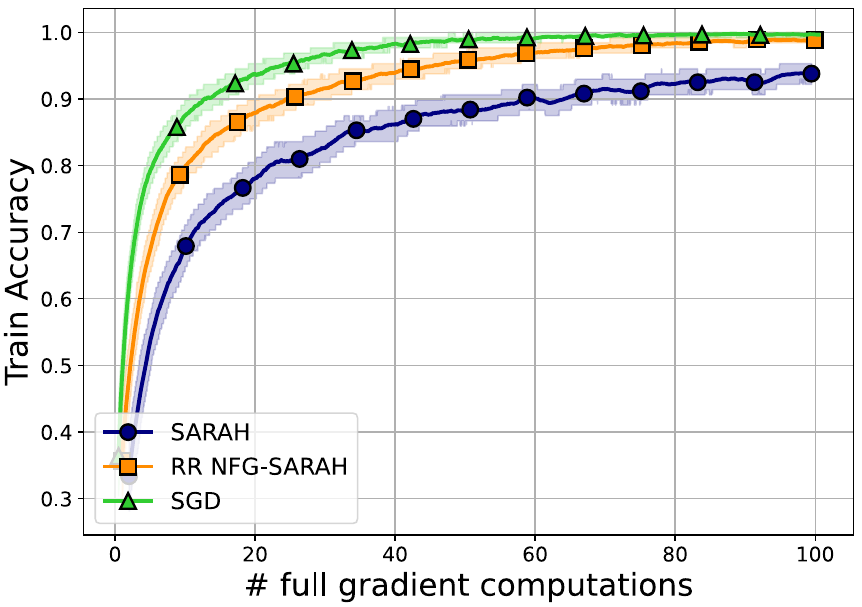}
\caption{train accuracy}
\end{subfigure}
\begin{subfigure}[b]{0.48\columnwidth}
\includegraphics[width=\textwidth]{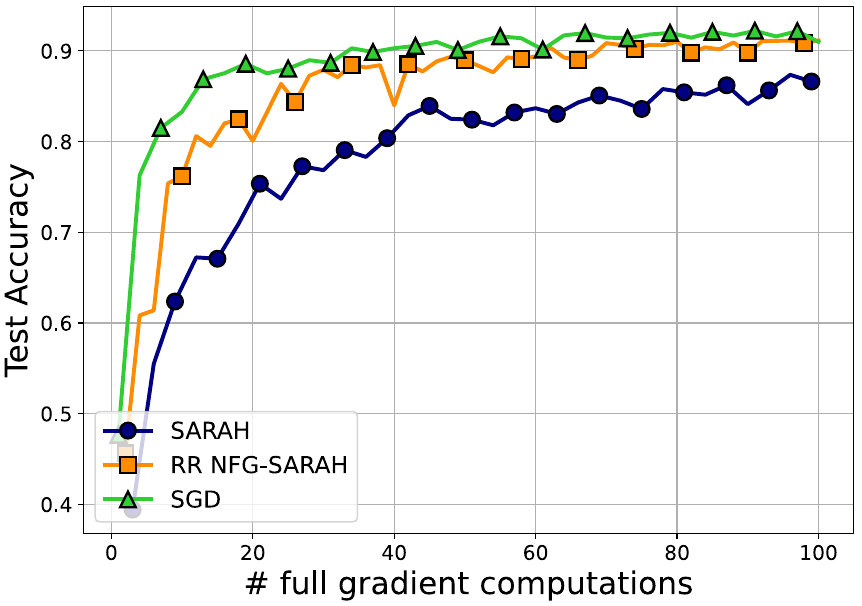}
\caption{test accuracy}
\end{subfigure}
\caption{\textsc{No Full Grad SARAH} and \textsc{SARAH}.}
\vspace{-4mm}
\label{fig:sarah_10}
\end{figure}
\textsc{NFG SARAH} ensures stable training loss convergence, surpassing the standard \textsc{SARAH} in speed concerning full gradient computations. On the test set, the loss reaches a comparable minimum to that of SGD but continues to decrease, while SGD begins to fluctuate. The final test loss is lower, and test accuracy progressively improves, with enhancement in the later stages.

We also extend the experimental validation of our methods. We present experiments on CIFAR-100 and fine-tune a Swin Transformer \citep{swin} on Tiny ImageNet \citep{tiny_imagenet} in Appendix \ref{sec:additionalexp}.

\section{CONCLUSION}
This paper introduces an approach that eliminates the necessity of full gradient computations in variance-reduced stochastic methods. Our technique approximates the full gradient via a moving average of stochastic gradients throughout an epoch, theoretically enabled by a shuffling heuristic. We establish upper convergence bounds by integrating this technique into both \textsc{SVRG} and \textsc{SARAH}. Furthermore, we provide lower bounds for the class of stochastic first-order methods employing shuffling.

While this work establishes both upper and lower complexity bounds, a complete picture requires closing the gap between them. Future research could aim to achieve this, for instance, through a refined analysis incorporating mini-batching strategies.

\section*{Acknowledgments}

The work was done in the Laboratory of Federated Learning Problems (Supported by Grant App. No. 2 to Agreement No. 075-03-2024-214).

\bibliography{aistats2026}
\bibliographystyle{plainnat}

\clearpage
\appendix
\thispagestyle{empty}

% Supplementary material: To improve readability, you must use a single-column format for the supplementary material.
\onecolumn
\aistatstitle{Variance Reduction Methods Do Not Need to Compute Full Gradients: Improved Efficiency through Shuffling}

\allowdisplaybreaks
%\tableofcontents

\section{ADDITIONAL EXPERIMENTS}\label{sec:additionalexp}

\subsection{Least squares regression.}
We consider the non-linear least squares loss problem:
%\vspace{-3mm}
\begin{equation} \label{eq:9}
   \textstyle{ f(x) = \frac{1}{n} \sum_{i=1}^{n} (y_i - h_i)^2,}
\end{equation} 
%\vspace{-5mm}
where $n$ is the number of samples, $y_i$ is the true value for sample $i$, $h_i$ is value for sample $i$, calculated as $ h_i = \frac{1}{1 + \exp(-z_i)}$, with $ z_i = A_i \cdot x $, addressing problem \eqref{eq:9}. Based on our theoretical estimates, which suggest inferior performance compared to standard \textsc{SVRG} and \textsc{SARAH}, we expect less favorable convergence. To address this limitation, we expand our investigation to examine the convergence of this method by tuning the stepsize, a topic that falls outside the scope of our current theoretical framework. The plots are shown in Figures \ref{fig:nfglog}-\ref{fig:nfglog2}.

\begin{figure}[H]
\centering
%\vspace{-0.4cm}
\begin{minipage}[][][b]{\columnwidth}
\centering
\resizebox{\columnwidth}{!}{%
\includegraphics[width=0.5\columnwidth]{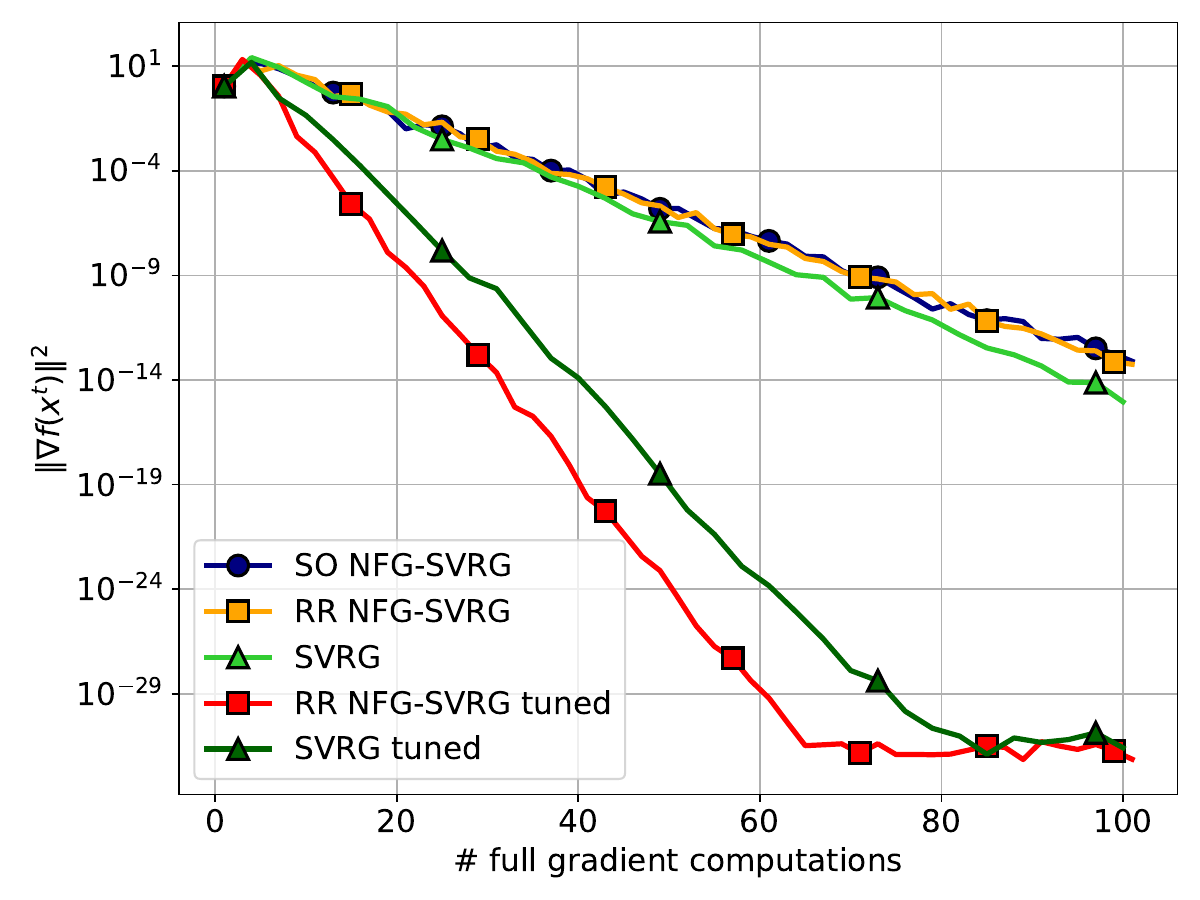}
\includegraphics[width=0.5\columnwidth]{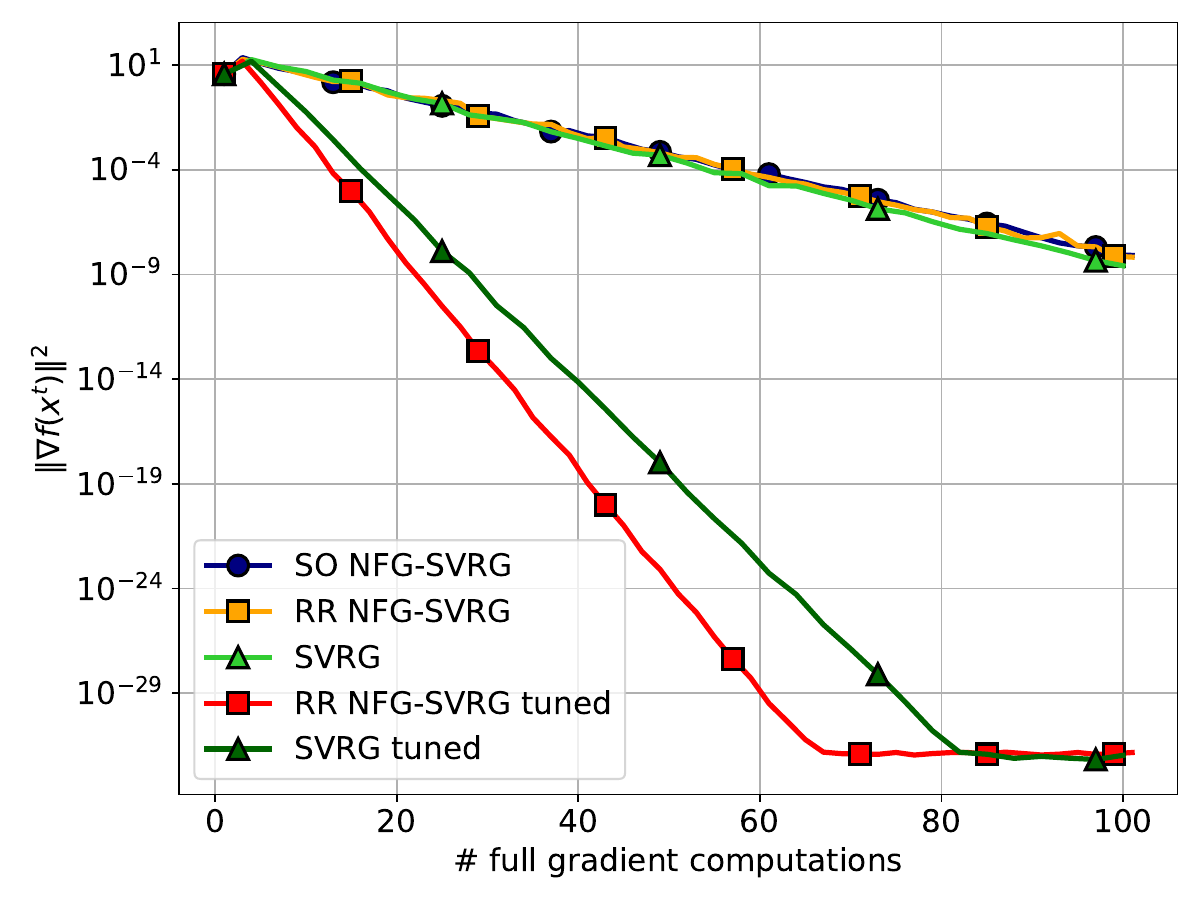}
}
\end{minipage}
\caption{\textsc{No Full Grad SVRG} and \textsc{SVRG} convergence with theoretical and tuned step sizes on problem \eqref{eq:9} on the \texttt{ijcnn1} (left) and \texttt{a9a} (right) datasets.}
\label{fig:nfglog}
\end{figure}

Upon examining the plots, we notice that although \textsc{No Full Grad} versions may converge slightly slower although comparable than its regular counterpart when using the theoretical step size, it significantly outperforms \textsc{SVRG} and \textsc{SARAH}, respectively, when the step size is optimally tuned. This highlights the potential of our method to achieve superior convergence rates with proper parameter adjustments, providing a robust alternative for large-scale optimization.

\begin{figure}[H]
\centering
%\vspace{-0.4cm}
\begin{minipage}[][][b]{\columnwidth}
\centering
\resizebox{\columnwidth}{!}{%
\includegraphics[width=0.5\columnwidth]{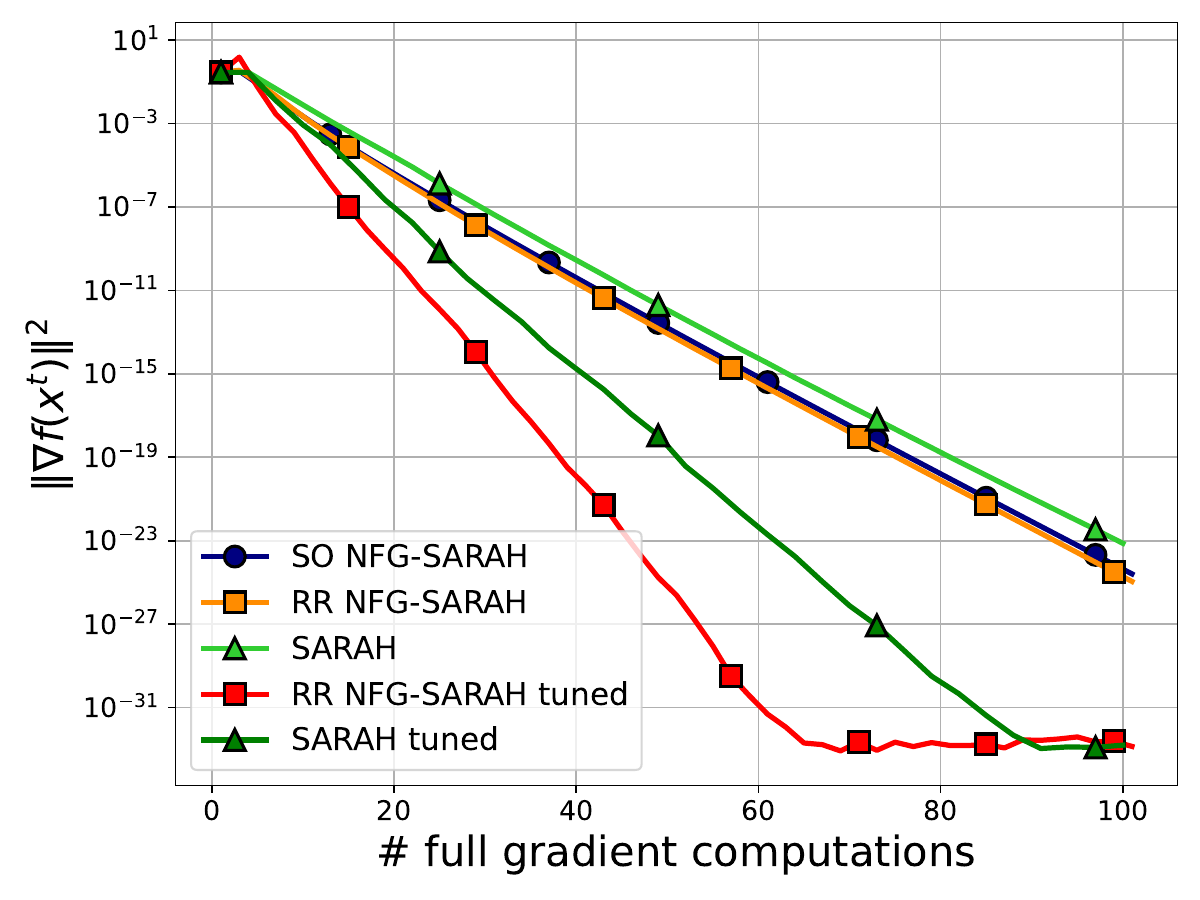}
\includegraphics[width=0.5\columnwidth]{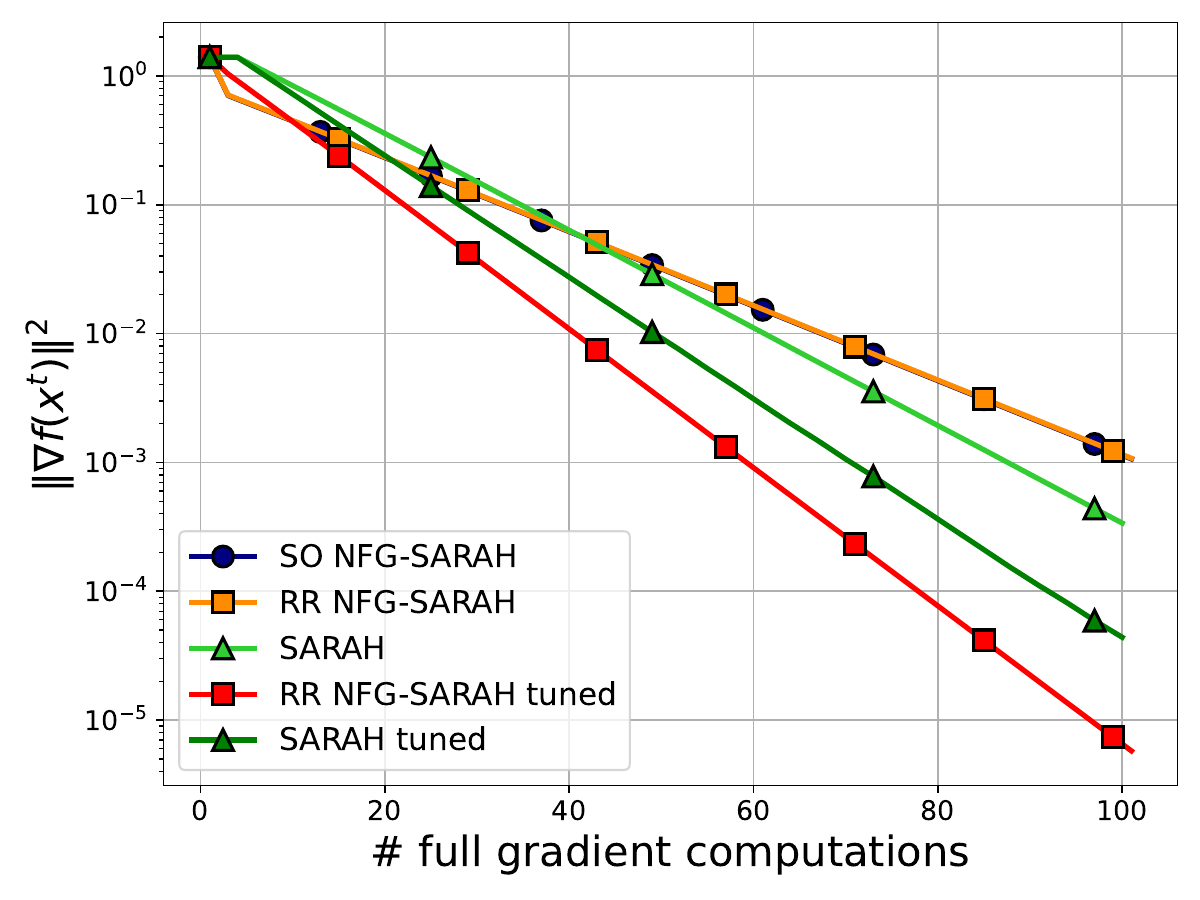}
}
\end{minipage}
\caption{\textsc{No Full Grad SARAH} and \textsc{SARAH} convergence with theoretical and tuned step sizes on problem \eqref{eq:9} on the \texttt{ijcnn1} (left) and \texttt{a9a} (right) datasets.}
\label{fig:nfglog2}
\end{figure}

\subsection{ResNet-18 on CIFAR-10/CIFAR-100 classification.}

\subsubsection*{Experiments on CIFAR-100}

We provide the results for image classification on the CIFAR-100 dataset. We keep the same experimental setup as for classification on the CIFAR-10 dataset (see Section \ref{sec:experiments}). The plots are provided in Figures \ref{fig:svrg_100}-\ref{fig:sarah_100}.

\begin{figure}[H]
\centering
\vspace{-0.4cm}
\begin{minipage}[][][b]{0.8\textwidth}
\centering
\resizebox{\columnwidth}{!}{%
\includegraphics[width=0.8\textwidth]{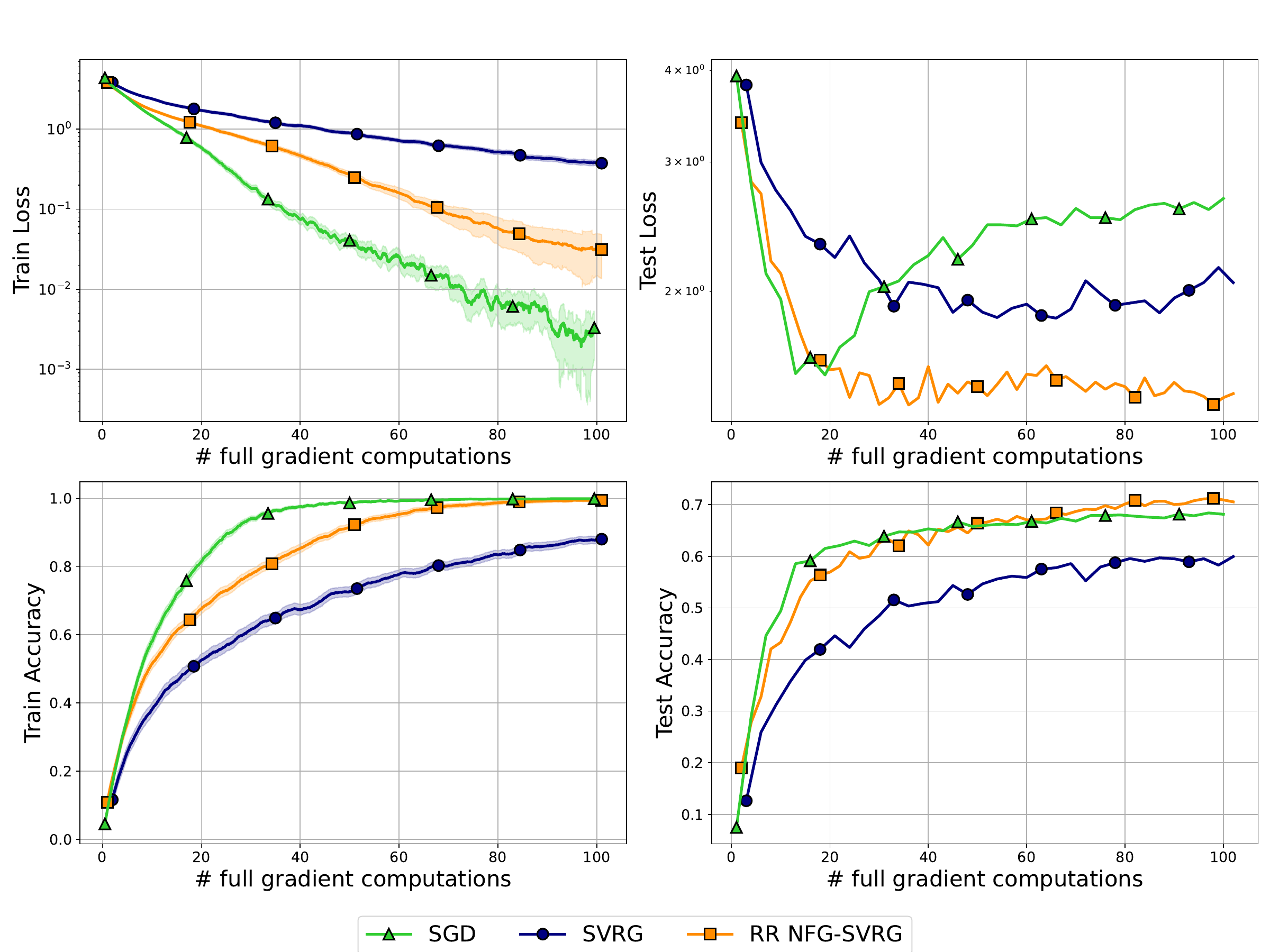}
}
\end{minipage}
% \vspace{-0.4cm}
\caption{\textsc{No Full Grad SVRG} and \textsc{SVRG}  on CIFAR-100 convergence.}
\label{fig:svrg_100}
\end{figure}

The SVRG algorithm follows a similar trend, with test loss stabilizing instead of increasing, unlike SGD. While SGD rebounds, SVRG maintains a plateau before further improvement. Test accuracy surpasses SGD from epoch 50 onward.

\begin{figure}[H]
\centering
\vspace{-0.4cm}
\begin{minipage}[][][b]{0.8\textwidth}
\centering
\resizebox{\columnwidth}{!}{%
\includegraphics[width=0.8\textwidth]{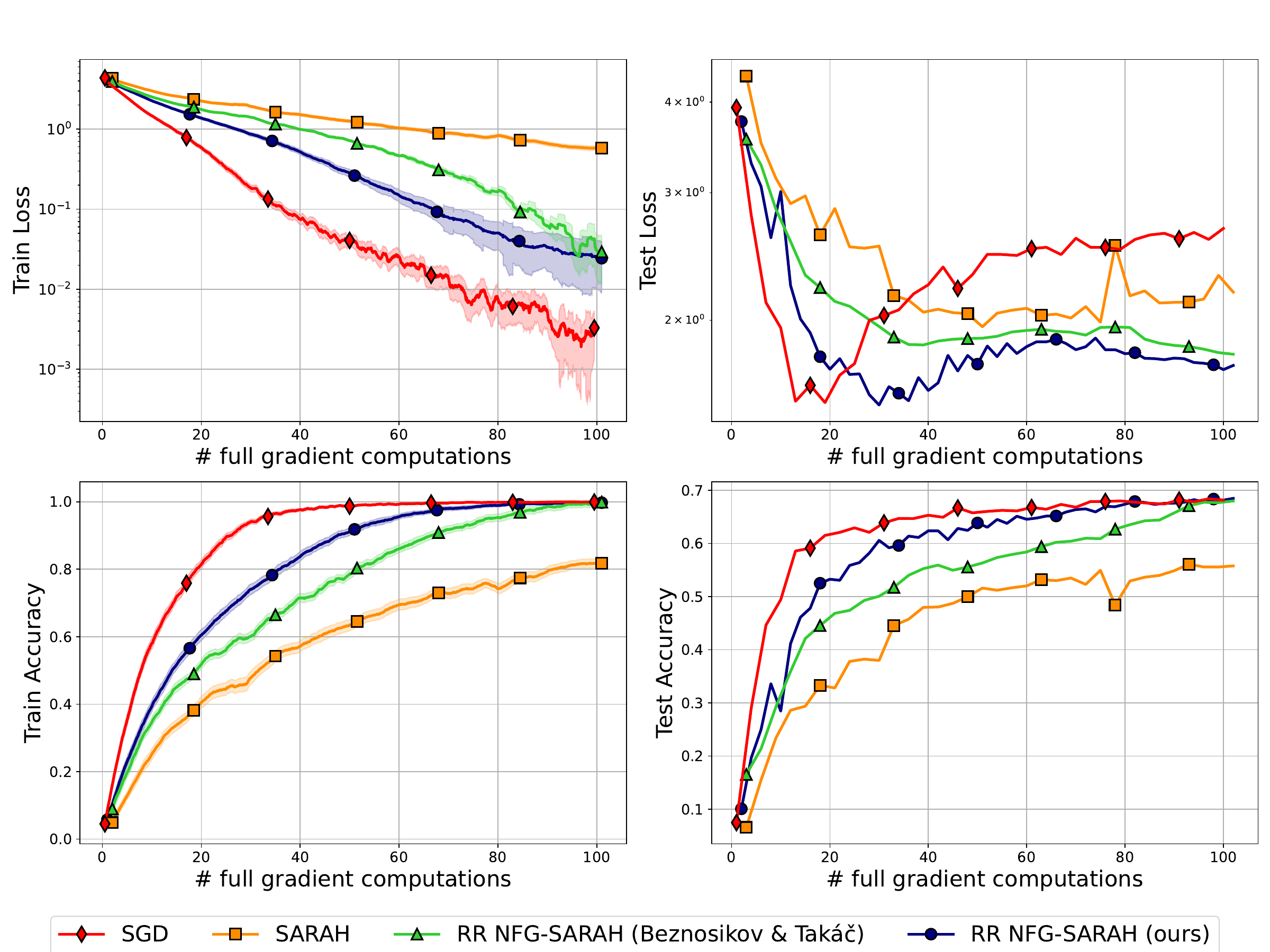}
}
\end{minipage}
% \vspace{-0.4cm}
\caption{\textsc{No Full Grad SARAH} and \textsc{SARAH} on CIFAR-100 convergence.}
\label{fig:sarah_100}
\end{figure}

The SARAH algorithm stabilizes training loss convergence and outpaces standard SARAH. Test loss decreases beyond SGD’s minimum, leading to a lower final loss. Test accuracy initially rises slowly but later accelerates, reducing overfitting.

\subsubsection*{Experimental Design}
 The experiments were implemented in Python using the PyTorch library \citep{paszke2019pytorch}, leveraging both a single CPU (Intel Xeon 2.20 GHz) and a single GPU (NVIDIA Tesla P100) for computation. To emulate a distributed environment, we split batches across multiple workers, simulating a decentralized optimization setting.

Our algorithms are evaluated in terms of accuracy and the number of full gradient computations. The experiments are conducted with the following setup:
\begin{itemize}
    \item number of workers $M = 5$;
    \item learning rate $\gamma = 0.1$ for both optimizers decaying to $10^{-3}$;
    \item regularization parameter $\lambda_1 = 0.0005$.
\end{itemize}

\subsection{Tiny ImageNet Classification with Swin Transformer fine-tuning}

\subsection*{Experimental Protocol}
Our image classification experiments on the Tiny ImageNet dataset \citep{tiny_imagenet} employed the Tiny Swin Transformer architecture \citep{swin}. This lightweight variant of the Swin Transformer is characterized by its hierarchical design and the use of shifted windows for efficient self-attention computation. The specific configuration utilized involved non-overlapping $4 \times 4$ input patches and a $7 \times 7$ window size for local self-attention.

We initialized the model using pretrained weights from ImageNet-1K \citep{imagenet}, specifically the \texttt{swin\_T\_patch4\_window7\_224} checkpoint provided in the official Swin Transformer repository\footnote{\url{https://github.com/microsoft/Swin-Transformer/blob/main/MODELHUB.md}}. The model was then fine-tuned on Tiny ImageNet.

The Tiny ImageNet dataset comprises 200 classes with images of $64 \times 64$ resolution. To meet the model's input requirements, all images were upsampled to $224 \times 224$. A standard ImageNet-style data augmentation pipeline was implemented, including random resized cropping and horizontal flipping.

Training spanned approximately 30 full gradient computations, with a batch size of 256. A cosine learning rate schedule was adopted, featuring a linear warm-up phase for the initial 10\% of total training steps, followed by decay to 10\% of the peak learning rate.
Weight decay was selected from $\{0, 0.01, 0.1\}$ based on validation performance. All optimization methods incorporated gradient clipping with a threshold of 1.0. 

\subsection*{Performance on Image Classification}
Further results and training curves for the Tiny Swin Transformer on the Tiny ImageNet classification task are presented in Figure \ref{fig:swin_plots}.

\begin{figure}[H]
\centering
\vspace{-0.4cm}
\begin{minipage}[][][b]{\textwidth}
\centering
\resizebox{\columnwidth}{!}{%
\includegraphics[width=\textwidth]{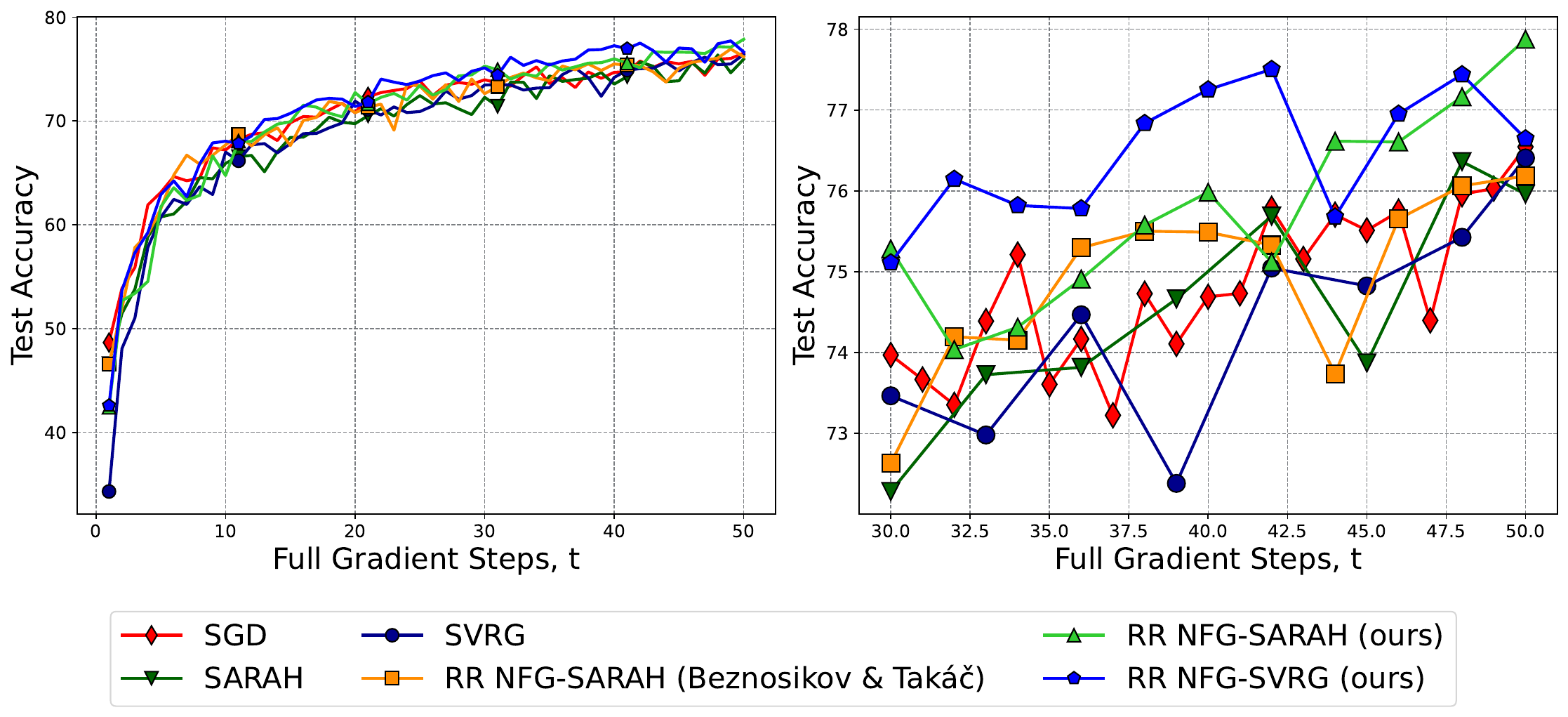}
}
\end{minipage}
% \vspace{-0.4cm}
\caption{\textsc{No Full Grad SARAH} and \textsc{SVRG} on Tiny ImageNet convergence.}
\label{fig:swin_plots}
\end{figure}

\begin{table}[ht]
\vspace{-8mm}
\centering
\caption{Final Accuracy of Variance Reduction Methods on Tiny ImageNet Convergence.}
\label{tab:vit}
\begin{tabular}{l|c}
\toprule
Algorithm & Final accuracy ($\uparrow$)\\
\midrule
\textsc{SGD} & 76.545  \\
\hline
\textsc{SARAH} & 75.961 \\
\hline
\textsc{SVRG} & 76.407 \\
\hline
\textsc{RR NFG-SARAH} \citep{beznosikov2023random} & 76.186 \\
\hline
\textsc{RR NFG-SARAH} (ours) & \textbf{77.875}\\
\hline
\textsc{RR NFG-SVRG} (ours) & 76.646 \\
\bottomrule
\end{tabular}
\end{table}

The results demonstrate the superior performance of our methods compared to classical variance reduction methods and modified version of \textsc{SARAH} \citep{beznosikov2023random}. The advantage is evident for both low-dimensional problems and complex networks with a large number of parameters.

\section{GENERAL INEQUALITIES}\label{sec:basicineq}
We introduce important inequalities that are used in further proofs. Let $f$ adhere to \text{Assumption~\ref{as1}}, $g$ adhere to \text{Assumption~\ref{as2stronglyconvex}}. Then for any real number $i$ and for all vectors $x, y, \{x_i\}\in\mathbb{R}^d$ with a positive scalars $\alpha, \beta$, the following inequalities hold:
\begin{align}
\label{ineq3} \tag{Scalar} 2\langle x, y \rangle & \leqslant \frac{\|x\|^2}{\alpha} + \alpha \|y\|^2, \\
\label{ineq:norm} \tag{Norm} 2\langle x, y \rangle & = \|x + y\|^2 - \|x\|^2 - \|y\|^, \\
\label{ineq:square} \tag{Quad} \|x + y\|^2 & \leqslant (1 + \beta)\|x\|^2 + (1 + \frac{1}{\beta})\|y\|^2, \\
\label{ineq4} \tag{Lip} f(x) & \leqslant f(y) + \langle \nabla f(y), x-y \rangle + \frac{L}{2} \|x-y\|^2,\\
\label{ineq1} \tag{CS}  \left\|\sum_{i=1}^{n} x_i\right\|^2 & \leqslant  n \sum_{i=1}^{n} \|x_i\|^2  \quad \quad (\text{Cauchy-Schwarz}),\\
\label{PL} \tag{PL}  g(x) - \inf g  &\leqslant \frac{1}{2\mu} \|\nabla g(x)\|^2 \quad \quad (\text{Polyak-Lojasiewicz}).
\end{align}
Here, \eqref{ineq4} was derived in \citep{nesterov2018lectures} in Theorem 2.1.5.

\section{NO FULL GRAD SVRG}\label{nfgsvrg_appendix}

For the convenience of the reader, we provide here the short description of Algorithm \ref{alg2}. If we consider it in epoch $s \neq 0$, one can note that the update rule is nothing but
\begin{align}
\label{svrg:update}
\begin{split}
\begin{cases}
    &\text{initial initialization:}\\
    &\quad\omega_s = x_s^0 = x_{s-1}^n\\
    &\quad v_s = \frac{1}{n}\sum\limits_{t = 0}^{n-1} \nabla f_{\pi_{s-1}^t} (x_{s-1}^t)\\
    &\text{for all iterations during the epoch}:\\
    &\quad v_s^t = \nabla f_{\pi_{s}^t} (x_s^t) - \nabla f_{\pi_{s}^t}(\omega_s) + v_s\\
    &\quad x_s^{t+1} = x_s^t - \gamma v_s^t
\end{cases}
\end{split}
\end{align}
\subsection{Non-convex setting}
\begin{lemma}\label{lemma1}
    Suppose that Assumptions \ref{as1}, \ref{as2} hold. Let the stepsize $\gamma \leqslant \frac{1}{Ln}$. Then for Algorithm \ref{alg2} it holds
    \begin{equation*}
        f(\omega_{s+1}) \leqslant f(\omega_s) - \frac{\gamma n}{2}\|\nabla f(\omega_s)\|^2 + \frac{\gamma n}{2}\left\|\nabla f(\omega_s) - \frac{1}{n}\sum\limits_{t=0}^{n-1} v_s^t\right\|^2.
    \end{equation*}
    \begin{proof}
    Using the iteration of Algorithm \ref{alg2} \eqref{svrg:update}, we have
        \begin{eqnarray*}
            f(\omega_{s+1}) &=& f(\omega_s - (\omega_s - \omega_{s+1})) \\
            &\overset{\eqref{ineq4}}{\leqslant}& f(\omega_s) + \langle \nabla f(\omega_s), \omega_{s+1} - \omega_s\rangle + \frac{L}{2}\|\omega_{s+1} - \omega_s\|^2 \\
            &=& f(\omega_s) - \gamma n\left\langle \nabla f(\omega_s), \frac{1} n\sum\limits_{t=0}^{n-1} v_s^t\right\rangle  + \frac{\gamma^2n^2L}{2}\left\|\frac{1}{n}\sum\limits_{t=0}^{n-1} v_s^t\right\|^2 \\
            &\overset{\eqref{ineq:norm}}{=}& f(\omega_s) - \frac{\gamma n}{2}\left[\|\nabla f(\omega_s)\|^2 + \left\|\frac{1}{n}\sum\limits_{t=0}^{n-1} v_s^t\right\|^2 - \left\|\nabla f(\omega_s) - \frac{1}{n}\sum\limits_{t=0}^{n-1} v_s^t\right\|^2 \right] \\
            & &+ \frac{\gamma^2n^2L}{2}\left\|\frac{1}{n}\sum\limits_{t=0}^{n-1} v_s^t\right\|^2  \\
            &=& f(\omega_s) - \frac{\gamma n}{2}\left[\|\nabla f(\omega_s)\|^2 - \left\|\nabla f(\omega_s) - \frac{1}{n}\sum\limits_{t=0}^{n-1} v_s^t\right\|^2 \right] \\
            & & - \frac{\gamma n}{2} \cdot \left(1 - \gamma nL \right)\left\|\frac{1}{n}\sum\limits_{t=0}^{n-1} v_s^t\right\|^2,
        \end{eqnarray*}
    Choosing \(\gamma : \frac{\gamma n}{2}\left(1 - \gamma nL\right) > 0 \) and note that such a choice is followed by \( \gamma\leqslant \frac{1}{Ln}\). In that way, we make the last term is negative and obtain the result of the lemma.  
    \end{proof}
\end{lemma}

Now we want to address the last term in the inequality of Lemma \ref{lemma1}. We prove the following lemma.

\begin{lemma}[\textbf{Lemma \ref{ngl2}}]\label{lemma2}
Suppose that Assumptions \ref{as1}, \ref{as2} hold. Then for Algorithm \ref{alg2} a valid estimate is
    \begin{align*}
    \left\| \nabla f(\omega_s) - \frac{1}{n}\sum\limits_{t=0}^{n-1} v_s^t\right\|^2 &\leqslant 2\|\nabla f(\omega_{s}) - v_s \|^2 + \frac{2L^2}{n}\sum\limits_{t=0}^{n-1} \|x_s^t - \omega_s\|^2.
    \end{align*}

\begin{proof}
We straightforwardly move to estimate of the desired norm:
\begin{align}
    \notag\left\| \nabla f(\omega_s) - \frac{1}{n}\sum\limits_{t=0}^{n-1} v_s^t\right\|^2 &\overset{\eqref{svrg:update}}{=} \left\|\nabla f(\omega_s) - \frac{1}{n}\left(nv_s + \sum\limits_{t=0}^{n-1}\left(\nabla f_{\pi_s^t}(x_s^t) - \nabla f_{\pi_s^t}(\omega_s)\right)\right)\right\|^2\\
    \notag & \overset{\eqref{ineq1}}{\leqslant} 2\|\nabla f(\omega_s) - v_s\|^2 + \frac{2}{n^2}\left\|\sum\limits_{t=0}^{n-1}\left(\nabla f_{\pi_s^t}(x_s^t) - \nabla f_{\pi_s^t}(\omega_s)\right)\right\|^2\\
    \notag & \overset{\eqref{ineq1}}{\leqslant} 2\|\nabla f(\omega_s) - v_s\|^2 + \frac{2}{n} \sum\limits_{t=0}^{n-1}\left\|\nabla f_{\pi_s^t}(x_s^t) - \nabla f_{\pi_s^t}(\omega_s)\right\|^2\\
    & \overset{\text{Ass. \ref{as1}}}{\leqslant} 2\|\nabla f(\omega_s) - v_s\|^2 + \frac{2L^2}{n} \sum\limits_{t=0}^{n-1}\left\|x_s^t - \omega_s\right\|^2,
\end{align}
which ends the proof.
\end{proof}    
\end{lemma}

\begin{lemma}[\textbf{Lemma \ref{ngl3}}]\label{lemma3}
Suppose that Assumptions \ref{as1}, \ref{as2} hold. Let the stepsize $\gamma \leqslant \frac{1}{2Ln}$. Then for Algorithm \ref{alg2} a valid estimate is
\begin{align*}
        \left\| \nabla f(\omega_s) - \frac{1}{n}\sum\limits_{t=0}^{n-1} v_s^t\right\|^2 & \leqslant 8\gamma^2L^2n^2\|v_{s}\|^2 + 32\gamma^2L^2n^2\|v_{s-1}\|^2.
\end{align*}
\begin{proof}
        To begin with, in Lemma \ref{lemma2}, we obtain
        \begin{equation}
        \label{l3:ineq1}
            \left\| \nabla f(\omega_s) - \frac{1}{n}\sum\limits_{t=0}^{n-1} v_s^t\right\|^2 \leqslant 2\|\nabla f(\omega_{s}) - v_s \|^2 + \frac{2L^2}{n}\sum\limits_{t=0}^{n-1} \|x_s^t - \omega_s\|^2.
        \end{equation}
        Let us show what $v_s$ is (here we use Line  \ref{alg2:line6} of Algorithm \ref{alg2}):
        \begin{align}
            \notag v_s & = \widetilde{v}_{s-1}^{n} = \frac{n-1}{n} \widetilde{v}_{s-1}^{n-1} + \frac{1}{n}\nabla f_{\pi_{s-1}^{n-1}} (x_{s-1}^{n-1}) \\
            \notag & = \frac{n-1}{n}\cdot\frac{n-2}{n-1} \widetilde{v}_{s-1}^{n-2} + \frac{n-1}{n}\cdot\frac{1}{n-1}\nabla f_{\pi_{s-1}^{n-2}}(x_{s-1}^{n-2})
            + \frac{1}{n}\nabla f_{\pi_{s-1}^{n-1}} (x_{s-1}^{n-1})\\
            \notag& = \frac{n-1}{n}\cdot\frac{n-2}{n-1}\cdot\ldots\cdot 0\cdot \widetilde{v}_{s-1}^0 + \frac{1}{n}\sum\limits_{t = 0}^{n-1}\nabla f_{\pi_{s-1}^t}(x_{s-1}^t)\\
            \label{l3:ineq2}& \overset{(i)}{=} \frac{1}{n}\sum\limits_{t = 0}^{n-1}\nabla f_{\pi_{s-1}^t}(x_{s-1}^t),
        \end{align}
        where equation (\textit{i}) is correct due to initialization $\widetilde{v}_{s-1}^0 = 0$ (Line \ref{alg2:line12} of Algorithm \ref{alg2}). In that way, using \eqref{l3:ineq1} and \eqref{l3:ineq2},
        \begin{eqnarray*}
            \left\| \nabla f(\omega_s) - \frac{1}{n}\sum\limits_{t=0}^{n-1} v_s^t\right\|^2 &\leqslant& 2\left\|\nabla f(\omega_s) - \frac{1}{n}\sum\limits_{t = 0}^{n-1}\nabla f_{\pi_{s-1}^t}(x_{s-1}^t) \right\|^2 \\
            & & + \frac{2L^2}{n}\sum\limits_{t=0}^{n-1} \|x_s^t - \omega_s\|^2.
        \end{eqnarray*}
        Then, using \eqref{eq:finite-sum},
        \begin{eqnarray}
            \notag \left\| \nabla f(\omega_s) - \frac{1}{n}\sum\limits_{t=0}^{n-1} v_s^t\right\|^2 &\leqslant& 2\left\|\frac{1}{n}\sum\limits_{t = 0}^{n-1}\left(\nabla f_{\pi_{s-1}^t}(\omega_{s}) - \nabla f_{\pi_{s-1}^t}(x_{s-1}^t) \right)\right\|^2 \\
            \notag& & + \frac{2L^2}{n}\sum\limits_{t=0}^{n-1} \|x_s^t - \omega_s\|^2 \\
            \notag &\overset{\eqref{ineq1}}{\leqslant}& \frac{2}{n}\sum\limits_{t=0}^{n-1}\|\nabla f_{\pi_{s-1}^t}(\omega_{s}) - \nabla f_{\pi_{s-1}^t}(x_{s-1}^t)\|^2 \\
            \notag & & + \frac{2L^2}{n}\sum\limits_{t=0}^{n-1} \|x_s^t - \omega_s\|^2 \\
            \notag &\overset{\text{Ass. \ref{as1}}}{\leqslant}& \frac{2L^2}{n}\sum\limits_{t=0}^{n-1}\|x_{s-1}^t - \omega_s\|^2 + \frac{2L^2}{n}\sum\limits_{t=0}^{n-1} \|x_s^t - \omega_s\|^2 \\
            \notag &\overset{\eqref{ineq:square}}{\leqslant}& \frac{4L^2}{n}\sum\limits_{t=0}^{n-1}\|x_{s-1}^t - \omega_{s-1}\|^2 + \frac{4L^2}{n}\sum\limits_{t=0}^{n-1}\|\omega_s - \omega_{s-1}\|^2  \\
            \label{l3:ineq3}& & + \frac{2L^2}{n}\sum\limits_{t=0}^{n-1} \|x_s^t - \omega_s\|^2.
        \end{eqnarray}
        Now we have to bound these three terms. Let us begin with $\sum\limits_{t=0}^{n-1} \|x_s^t - \omega_s\|^2$.
        \begin{eqnarray*}
            \sum\limits_{t=0}^{n-1} \|x_s^t - \omega_s\|^2 &=& \gamma^2\sum\limits_{t=0}^{n-1} \left\|\sum\limits_{k = 0}^{t-1} v_s^k\right\|^2 \overset{\eqref{svrg:update}}{=}\gamma^2\sum\limits_{t=0}^{n-1} \left\|tv_s + \sum\limits_{k = 0}^{t-1} \left(\nabla f_{\pi_s^k}(x_s^k) - \nabla f_{\pi_s^k}(\omega_s)\right)\right\|^2\\
            &\overset{\eqref{ineq1}}{\leqslant}& 2\gamma^2\sum\limits_{t=0}^{n-1}t^2\|v_s\|^2 + 2\gamma^2\sum\limits_{t=0}^{n-1}t\sum\limits_{k = 0}^{t-1}\|\nabla f_{\pi_s^k}(x_s^k) - \nabla f_{\pi_s^k}(\omega_s)\|^2\\
            &\overset{\text{Ass. \ref{as1}}}{\leqslant}& 2\gamma^2n^3\|v_s\|^2 + 2\gamma^2L^2n\sum\limits_{t=0}^{n-1}\sum\limits_{k=0}^{t-1}\|x_s^k - \omega_s\|^2\\
            &\leqslant& 2\gamma^2n^3\|v_s\|^2 + 2\gamma^2L^2n^2\sum\limits_{t=0}^{n-2}\|x_s^t - \omega_s\|^2\\
            &\leqslant& 2\gamma^2n^3\|v_s\|^2 + 2\gamma^2L^2n^2\sum\limits_{t=0}^{n-1}\|x_s^t - \omega_s\|^2.
        \end{eqnarray*}
        Expressing $\sum\limits_{t = 0}^{n-1}\|x_s^t - \omega_s\|^2$ from here, we get
        \begin{equation*}
            \sum\limits_{t = 0}^{n-1}\|x_s^t - \omega_s\|^2 \leqslant \frac{2\gamma^2 n^3\|v_s\|^2}{1 - 2\gamma^2L^2n^2}.
        \end{equation*}
         To finish this part of proof it remains for us to choose appropriate $\gamma$. In Lemma \ref{lemma1} we require $\gamma \leqslant\frac{1}{Ln}$. There we choose smaller values of $\gamma: \gamma\leqslant\frac{1}{2Ln}$ (with that values all previous transitions is correct). Now we provide final estimation of this norm:
        \begin{align}
        \label{l3:ineq4}
            \sum\limits_{t=0}^{n-1}\|x_s^t - \omega_s\|^2 \leqslant 4\gamma^2n^3\|v_s\|^2.
        \end{align}
        One can note the boundary of the $\sum\limits_{t = 0}^{n-1}\|x_{s-1}^t - \omega_{s-1}\|^2$ term is similar because it involves the same sum of norms from the previous epoch.
        \begin{align}
        \label{l3:ineq5}
            \sum\limits_{t = 0}^n\|x_{s-1}^t - \omega_{s-1}\|^2 \leqslant 4\gamma^2n^3\|v_{s-1}\|^2.
        \end{align}
        It remains for us to estimate the $\sum\limits_{t=0}^{n-1}\|\omega_s - \omega_{s-1}\|^2$ term.
        \begin{eqnarray*}
            \sum\limits_{t=0}^{n-1} \|\omega_s - \omega_{s-1}\|^2 &=& \gamma^2\sum\limits_{t=0}^{n-1} \left\|\sum\limits_{k = 0}^{n-1} v_{s-1}^k\right\|^2 \\
            &\overset{\eqref{svrg:update}}{=}& \gamma^2\sum\limits_{t=0}^{n-1} \left\|nv_{s-1} + \sum\limits_{k = 0}^{n-1} \left(\nabla f_{\pi_{s-1}^k}(x_{s-1}^k) - \nabla f_{\pi_{s-1}^k}(\omega_{s-1})\right)\right\|^2\\
            &\overset{\eqref{ineq1}}{\leqslant}& 2\gamma^2\sum\limits_{t=0}^{n-1}n^2\|v_{s-1}\|^2 \\
            & & + 2\gamma^2\sum\limits_{t=0}^{n-1}n\sum\limits_{k = 0}^{n-1}\|\nabla f_{\pi_{s-1}^k}(x_{s-1}^k) - \nabla f_{\pi_{s-1}^k}(\omega_{s-1})\|^2\\
            &\overset{\text{Ass. \ref{as1}}}{\leqslant}& 2\gamma^2n^3\|v_{s-1}\|^2 + 2\gamma^2L^2n\sum\limits_{t=0}^{n-1}\sum\limits_{k=0}^{t-1}\|x_{s-1}^k - \omega_{s-1}\|^2\\
            &\leqslant& 2\gamma^2n^3\|v_{s-1}\|^2 + 2\gamma^2L^2n^2\sum\limits_{t=0}^{n-2}\|x_{s-1}^t - \omega_{s-1}\|^2\\
            &\leqslant& 2\gamma^2n^3\|v_{s-1}\|^2 + 2\gamma^2L^2n^2\sum\limits_{t=0}^{n-1}\|x_{s-1}^t - \omega_{s-1}\|^2\\
            &\overset{\eqref{l3:ineq5}}{\leqslant}& 2\gamma^2n^3\|v_{s-1}\|^2 + 8\gamma^4L^2n^5\sum\limits_{t=0}^{n-1}\|x_{s-1}^t - \omega_{s-1}\|^2.
        \end{eqnarray*}
        Using our choice $\gamma\leqslant\frac{1}{2Ln}$, we derive the estimate of the last term:
        \begin{align}\label{l3:ineq6}
            \sum\limits_{t=0}^{n-1} \|\omega_s - \omega_{s-1}\|^2 \leqslant 2\gamma^2n^3\|v_{s-1}\|^2 + 2\gamma^2n^3\|v_{s-1}\|^2 = 4\gamma^2n^3\|v_{s-1}\|^2.
        \end{align}
        Now we can apply the upper bounds obtained in \eqref{l3:ineq4} -- \eqref{l3:ineq6} to \eqref{l3:ineq3} and have
        \begin{align*}
            \left\| \nabla f(\omega_s) - \frac{1}{n} \sum\limits_{t=0}^{n-1} v_s^t\right\|^2 & \leqslant 8\gamma^2L^2n^2\|v_{s}\|^2 + 16\gamma^2L^2n^2\|v_{s-1}\|^2 + 16\gamma^2L^2n^2\|v_{s-1}\|^2 \\
            & = 8\gamma^2L^2n^2\|v_{s}\|^2 + 32\gamma^2L^2n^2\|v_{s-1}\|^2,
        \end{align*}
        which ends the proof.
        \end{proof}
\end{lemma}

\begin{theorem}[\textbf{Theorem \ref{nfgt1}}]
Suppose Assumptions \ref{as1}, \ref{as2nonconvex} hold. Then Algorithm \ref{alg2} with $\gamma\leqslant\frac{1}{20L n}$ to reach $\varepsilon$-accuracy, where $\varepsilon^2 = \frac{1}{S}\sum\limits_{s=1}^{S} \|\nabla f(\omega_s)\|^2$, needs
    \begin{equation*}
        \mathcal{O} \left(\frac{nL}{\varepsilon^2}\right)~~ \text{iterations and oracle calls.}
    \end{equation*}
               \begin{proof}
        We combine the result of Lemma \ref{lemma1} with the result of Lemma \ref{lemma3} and obtain
        \begin{align*}
            f(\omega_{s+1}) &\leqslant f(\omega_s) - \frac{\gamma n}{2}\|\nabla f(\omega_s)\|^2 \\
            & \quad + \frac{\gamma n}{2}\left(8\gamma^2L^2n^2\|v_{s}\|^2 + 32\gamma^2L^2n^2\|v_{s-1}\|^2\right).
        \end{align*}
        We subtract $f(x^*)$ from both parts:
        \begin{align*}
            f(\omega_{s+1}) - f(x^*) &\leqslant f(\omega_s) - f(x^*) - \frac{\gamma n}{2}\|\nabla f(\omega_s)\|^2 \\
            & \quad + \frac{\gamma n}{2}\left(8\gamma^2L^2n^2\|v_{s}\|^2 + 32\gamma^2L^2n^2\|v_{s-1}\|^2\right)\\
            & = f(\omega_s) - f(x^*) - \frac{\gamma n}{4}\|\nabla f(\omega_s)\|^2 \\
            & \quad + \frac{\gamma n}{2}\left(8\gamma^2L^2n^2\|v_{s}\|^2 + 32\gamma^2L^2n^2\|v_{s-1}\|^2\right) \\
            & \quad - \frac{\gamma n}{4}\|\nabla f(\omega_s)\|^2.
            \end{align*}
       Then, transforming the last term by using \eqref{ineq:square} with $\beta=1$, we get
        \begin{align*}
            f(\omega_{s+1}) - f(x^*) &\leqslant f(\omega_s) - f(x^*) - \frac{\gamma n}{4}\|\nabla f(\omega_s)\|^2 \\
           & \quad  + \frac{\gamma n}{2}\left(8\gamma^2L^2n^2\|v_{s}\|^2 + 32\gamma^2L^2n^2\|v_{s-1}\|^2\right) \\
            & \quad- \frac{\gamma n}{8}\|v_s\|^2 + \frac{\gamma n}{4}\|v_s - \nabla f(x
            \omega_s)\|^2.
            \end{align*}
        Using Lemma \ref{lemma3} to $\|v_s - \nabla f(\omega_s)\|^2$ (specially $\frac{4L^2}{n}\cdot\eqref{l3:ineq5} + \frac{4L^2}{n}\cdot\eqref{l3:ineq6}$),
        \begin{align*}
            f(\omega_{s+1}) - f(x^*) &\leqslant f(\omega_s) - f(x^*) - \frac{\gamma n}{4}\|\nabla f(\omega_s)\|^2 \\
            & \quad + \frac{\gamma n}{2}\left(8\gamma^2L^2n^2\|v_{s}\|^2 + 32\gamma^2L^2n^2\|v_{s-1}\|^2\right) \\
            & \quad- \frac{\gamma n}{8}\|v_s\|^2 + \frac{\gamma n}{4}\cdot 32\gamma^2L^2n^2\|v_{s-1}\|^2.
        \end{align*}
        Combining alike expressions,
        \begin{align}
            \notag f(\omega_{s+1}) - f(x^*) + \frac{\gamma n}{4}\|\nabla f(\omega_s)\|^2 &\leqslant f(\omega_s) - f(x^*) - \frac{\gamma n}{8}\left(1 - 32 \gamma^2L^2n^2\right)\|v_s\|^2 \\ 
            \label{t1:ineq1}& \quad + \gamma n\cdot 24\gamma^2L^2n^2\|v_{s-1}\|^2. 
        \end{align}
        Using $\gamma \leqslant \frac{1}{20Ln}$ (note it is the smallest stepsize from all the steps we used before, so all previous transitions are correct), we get
        \begin{align*}
            f(\omega_{s+1}) - f(x^*) &+ \frac{1}{10}\gamma n\|v_s\|^2 + \frac{\gamma(n+1)}{4}\|\nabla f(\omega_s)\|^2\\
            &\leqslant f(\omega_s) - f(x^*) + \frac{1}{10}\gamma n\|v_{s-1}\|^2.
        \end{align*}
        Next, denoting $\Delta_{s} = f(\omega_{s+1}) - f(x^*) + \frac{1}{10}\gamma n\|v_{s}\|^2$, we obtain
        \begin{equation*}
            \frac{1}{S}\sum\limits_{s=1}^{S} \|\nabla f(\omega_s)\|^2 \leqslant \frac{4\left[\Delta_0 - \Delta_{S}\right]}{\gamma nS}.
        \end{equation*}    
        We choose $\varepsilon^2 = \frac{1}{S}\sum\limits_{s=1}^{S} \|\nabla f(\omega_s)\|^2$ as criteria. Hence, to reach $\varepsilon$-accuracy we need $\mathcal{O}\left(\frac{L}{\varepsilon^2}\right)$ epochs and $\mathcal{O}\left(\frac{nL}{\varepsilon^2}\right)$ iterations. Additionally, we note that the oracle complexity of our algorithm is also equal to $\mathcal{O}(\frac{nL}{\varepsilon^2})$, since at each iteration the algorithm computes the stochastic gradient at only two points. This ends the proof.
\end{proof}
\end{theorem}

%%%%%%%%%%%%%%%%%%%%%%%%%
\subsection{Strongly convex setting} 
\begin{theorem}[\textbf{Theorem \ref{nfgt2}}]\label{theorem2}
Suppose Assumptions \ref{as1}, \ref{as2stronglyconvex} hold. Then Algorithm \ref{alg2} with $\gamma\leqslant\frac{1}{20Ln}$ to reach $\varepsilon$-accuracy, where $\varepsilon = f(x_{S+1}^0)-f(x^*)$, needs
   %\vspace{-3mm}
    \begin{equation*}
        \mathcal{O} \left(\frac{nL}{\mu}\log \frac{1}{\varepsilon}\right)~~ \text{iterations and oracle calls.}
    \end{equation*}
        
\begin{proof}
Under Assumption \ref{as2stronglyconvex}, which states that the function is strongly convex, the \eqref{PL} condition is automatically satisfied. Therefore, 
\begin{align*}
    f(\omega_{s+1}) - f(x^*) + \frac{\gamma\mu n}{2}\left( f(\omega_s) - f(x^*)\right) &\leqslant  f(\omega_{s+1}) - f(x^*) + \frac{\gamma n}{4}\|\nabla f(\omega_s)\|^2.
\end{align*}
Thus, using \eqref{t1:ineq1},
\begin{align*}
f(\omega_{s+1}) - f(x^*) &+ \frac{\gamma\mu n}{2}\left( f(\omega_s) - f(x^*)\right) \leqslant f(\omega_s) - f(x^*) \\
& - \frac{\gamma n}{8}\left(1 - 32\gamma^2L^2n^2\right)\|v_s\|^2 + \gamma n\cdot 24\gamma^2L^2n^2\|v_{s-1}\|^2.
\end{align*}
Using $\gamma \leqslant \frac{1}{20L(n+1)}$ and assuming $n \geqslant 2$, we get
\begin{align*}
    f(\omega_{s+1}) - f(x^*) + \frac{1}{10}\gamma n\|v_s\|^2
    &\leqslant \left(1-\frac{\gamma\mu n}{2}\right)\left(f(\omega_s) - f(x^*)\right) \\
    & \quad + \frac{1}{10}\gamma n\cdot \left(1-\frac{\gamma\mu n}{2}\right)\|v_{s-1}\|^2.
\end{align*}
Next, denoting $\Delta_{s} = f(\omega_{s+1}) - f(x^*) + \frac{1}{10}\gamma n\|v_{s}\|^2$, we obtain the final convergence over one epoch:
\begin{align*}
    \Delta_{s+1} \leqslant \left(1-\frac{\gamma\mu n}{2}\right) \Delta_s.
\end{align*}
Going into recursion over all epoch,
\begin{align*}
    f(\omega_{S+1}) - f(x^*)\leqslant\Delta_{S} \leqslant \left(1-\frac{\gamma\mu n}{2}\right)^{S+1}\Delta_0.
\end{align*}
We choose $\varepsilon = f(\omega_{S+1})-f(x^*)$ as criteria. Then to reach $\varepsilon$-accuracy we need $\mathcal{O}\left(\frac{L}{\mu}\log\left(\frac{1}{\varepsilon}\right)\right)$ epochs and $\mathcal{O}\left(\frac{nL}{\mu}\log\left(\frac{1}{\varepsilon}\right)\right)$ iterations. Additionally, we note that the oracle complexity of our algorithm is also equal to $\mathcal{O}\left(\frac{nL}{\mu}\log\left(\frac{1}{\varepsilon}\right)\right)$, since at each iteration the algorithm computes the stochastic gradient at only two points.
\end{proof} 
\end{theorem}
    
\section{NO FULL GRAD SARAH}\label{nfgsarah_appendix}

For the convenience of the reader, we provide here the short description of Algorithm \ref{alg:sarah}. If we consider it in epoch $s \neq 0$, one can note that the update rule is nothing but
\begin{align}
\label{sarah:update}
\begin{split}
\begin{cases}
    &\text{if iteration~} t=0:\\
    &\quad x_s^0 = x_{s-1}^{n}\\
    &\quad v_s^0 = v_s = \frac{1}{n}\sum\limits_{t = 1}^n \nabla f_{\pi_{s-1}^t} (x_{s-1}^t)\\
    &\quad x_s^1 = x_s^0 - \gamma v_s^0\\
    &\text{for rest iterations during the epoch}:\\
    &\quad v_s^t = v_s^{t-1} + \frac{1}{n} \left(\nabla f_{\pi_{s}^t} (x_s^t) - \nabla f_{\pi_{s}^t} (\omega_s)\right)\\
   &\quad x_s^{t+1} = x_s^t - \gamma v_s^t
\end{cases}
\end{split}
\end{align}
\subsection{Non-convex setting}
\begin{lemma}\label{lemma4}
    Suppose that Assumptions \ref{as1}, \ref{as2} hold. Let the stepsize $\gamma \leqslant \frac{1}{L(n+1)}$. Then for Algorithm \ref{alg:sarah} it holds
    \begin{equation*}
        f(x_{s+1}^0) \leqslant f(x_s^0) - \frac{\gamma(n+1)}{2}\|\nabla f(x_s^0)\|^2 + \frac{\gamma(n+1)}{2}\left\|\nabla f(x_s^0) - \frac{1}{n+1}\sum\limits_{i=0}^n v_s^i\right\|^2.
    \end{equation*}
    \begin{proof}
    Using the iteration of Algorithm \ref{alg:sarah} \eqref{sarah:update}, we have
        \begin{eqnarray*}
            f(x_{s+1}^0) &=& f(x_s^0 - (x_s^0 - x_{s+1}^0)) \\
            &\overset{\eqref{ineq4}}{\leqslant}& f(x_s^0) + \langle \nabla f(x_s^0), x_{s+1}^0 - x_s^0\rangle + \frac{L}{2}\|x_{s+1}^0 - x_s^0|^2 \\
            &=& f(x_s^0) - \gamma(n+1)\left\langle \nabla f(x_s^0), \frac{1}{n+1}\sum\limits_{t=0}^n v_s^t\right\rangle  + \frac{\gamma^2(n+1)^2L}{2}\left\|\frac{1}{n+1}\sum\limits_{t=0}^n v_s^t\right\|^2 \\
            &\overset{\eqref{ineq:norm}}{=}& f(x_s^0) - \frac{\gamma(n+1)}{2}\left[\|\nabla f(x_s^0)\|^2 + \left\|\frac{1}{n+1}\sum\limits_{t=0}^n v_s^t\right\|^2 \right.\\
            & & \left. - \left\|\nabla f(x_s^0) - \frac{1}{n+1}\sum\limits_{t=0}^n v_s^t\right\|^2 \right] + \frac{\gamma^2(n+1)^2L}{2}\left\|\frac{1}{n+1}\sum\limits_{t=0}^n v_s^t\right\|^2  \\
            &=& f(x_s^0) - \frac{\gamma(n+1)}{2}\left[\|\nabla f(x_s^0)\|^2 - \left\|\nabla f(x_s^0) - \frac{1}{n+1}\sum\limits_{t=0}^n v_s^t\right\|^2 \right] \\
            & & - \frac{\gamma(n+1)}{2} \cdot \left(1 - \gamma(n+1)L \right)\left\|\frac{1}{n+1}\sum\limits_{t=0}^n v_s^t\right\|^2.
        \end{eqnarray*}
    It remains for us to choose \(\gamma : \frac{\gamma(n+1)}{2}\left(1 - \gamma(n+1)L\right) > 0 \) and note that such a choice is followed by \( \gamma\leqslant \frac{1}{L(n+1)}\). In that way we make the last term is negative and obtain the result of the lemma. 
    \end{proof}
\end{lemma}

Now we want to address the last term in the result of Lemma \ref{lemma4}. We prove the following lemma.

\begin{lemma}[\textbf{Lemma \ref{l1:sarahmain}}]\label{lemma5}
Suppose that Assumptions \ref{as1}, \ref{as2} hold. Then for Algorithm \ref{alg:sarah} a valid estimate is
\begin{align*}
    \left\| \nabla f(x_s^0) - \frac{1}{n+1}\sum\limits_{t=0}^n v_s^t\right\|^2 &\leqslant 2\|\nabla f(x_{s}^0) - v_s \|^2 + \frac{2L^2}{n+1}\sum\limits_{t=1}^n \|x_s^t - x_s^{t-1}\|^2.
\end{align*}

\begin{proof}
    We claim that
    \begin{equation}
    \label{l5:ineq1}
        \sum\limits_{t=k}^n v_s^t = \frac{1}{n} \sum\limits_{t = k+1}^n (n-t+1)\left(\nabla f_{\pi_s^t}(x_s^t) - \nabla f_{\pi_s^t}(x_s^{t-1})\right) + (n-k+1)v_s^k.
    \end{equation}
    Let us prove this. We use the method of induction. For \( k = n \) it is obviously true. We suppose that it is true for some some fixed $k = \widetilde{k} \geqslant 1$ ($k=0$ is the first index in the epoch, i.e. start of the recursion) and want to prove that it is true for $k = \widetilde{k}-1$.
    \begin{align*}
        \sum\limits_{t=\widetilde{k}-1}^n v_s^t &= v_s^{\widetilde{k}-1} + \sum\limits_{t=\widetilde{k}}^n v_s^t \\
        &= v_s^{\widetilde{k}-1} + \frac{1}{n} \sum\limits_{t = \widetilde{k}+1}^n (n-t+1)\left(\nabla f_{\pi_s^t}(x_s^t) - \nabla f_{\pi_s^t}(x_s^{t-1})\right) + (n-\widetilde{k}+1)v_s^{\widetilde{k}} \\
        & \overset{(i)}{=}  v_s^{\widetilde{k}-1} + \frac{1}{n} \sum\limits_{t = \widetilde{k}+1}^n (n-t+1)\left(\nabla f_{\pi_s^t}(x_s^t) - \nabla f_{\pi_s^t}(x_s^{t-1})\right) \\
        & \quad + (n-\widetilde{k}+1)\left(v_s^{\widetilde{k}-1} + \frac{1}{n} \left(\nabla f_{\pi_s^{\widetilde{k}}}(x_s^{\widetilde{k}}) - \nabla f_{\pi_s^{\widetilde{k}}}(x_s^{\widetilde{k}-1})\right)\right) \\
        &= \frac{1}{n} \sum\limits_{t = \widetilde{k}}^n (n-t+1)\left(\nabla f_{\pi_s^t}(x_s^t) - \nabla f_{\pi_s^t}(x_s^{t-1})\right) + (n-\widetilde{k}+2) v_s^{\widetilde{k}-1},
    \end{align*}
    where equation (\textit{i}) is correct due to \eqref{sarah:update} and $\widetilde{k} \geqslant 1$. In that way, the induction step is proven. It means that \eqref{l5:ineq1} is valid. We substitute $ k = 0 $ in \eqref{l5:ineq1} and, utilizing $v_s^0 = v_s$, get
    \begin{align}
    \label{l5:ineq2}
        \sum\limits_{t=0}^n v_s^t &= \frac{1}{n} \sum\limits_{t = 1}^n (n-t+1)\left(\nabla f_{\pi_s^t}(x_s^t) - \nabla f_{\pi_s^t}(x_s^{t-1})\right) + (n+1) v_s.
    \end{align}
    Hence, estimating the desired term gives
    \begin{eqnarray*}
        \Biggl\|\nabla f(x_s^0) - \frac{1}{n+1}\sum\limits_{t=0}^n v_s^t\Biggr\|^2 
        &=& \frac{1}{(n+1)^2}\left\|(n+1)\nabla f(x_{s}^0) -\sum\limits_{t=0}^n v_s^t\right\|^2\\
        &\overset{\eqref{l5:ineq2}}{=}& \frac{1}{(n+1)^2}\Biggl\|(n+1)\nabla f(x_{s}^0) \\
        & & - \frac{1}{n}\sum\limits_{t = 1}^n (n-t+1)\left(\nabla f_{\pi_s^t}(x_s^t) - \nabla f_{\pi_s^t}(x_s^{t-1})\right) \\
        & & - (n+1)v_s\Biggr\|^2 \\
        &\overset{\eqref{ineq1}}{\leqslant}& 2\|\nabla f(x_{s}^0) - v_s\|^2 \\
        & & + \frac{2}{(n+1)^2}\left\|\frac{1}{n} \sum\limits_{t = 1}^n (n-t+1)\left(\nabla f_{\pi_s^t}(x_s^t) - \nabla f_{\pi_s^t}(x_s^{t-1})\right)\right\|^2\\
        &\overset{(i)}{\leqslant}& 2\|\nabla f(x_{s}^0) - v_s\|^2 \\
        & & + \frac{2}{(n+1)^2}\left\|\sum\limits_{t=1}^n \left(\nabla f_{\pi_s^t}(x_s^t) - \nabla f_{\pi_s^t}(x_s^{t-1})\right)\right\|^2\\
        &\overset{\eqref{ineq1}}{\leqslant}& 2\|\nabla f(x_{s}^0) - v_s\|^2 + \frac{2}{n+1}\sum\limits_{t=1}^n\left\|\nabla f_{\pi_s^t}(x_s^t) - \nabla f_{\pi_s^t}(x_s^{t-1})\right\|^2\\
        &\overset{\text{Ass. \ref{as1}}}{\leqslant}& 2\|\nabla f(x_{s}^0) - v_s\|^2 + \frac{2L^2}{n+1}\sum\limits_{t=1}^n\left\|x_s^t - x_s^{t-1}\right\|^2,
    \end{eqnarray*}
    where inequality (\textit{i}) is correct due to $t\geqslant 1$ holds during the summation. The obtained inequality finishes the proof of the lemma.
\end{proof}    
\end{lemma}

\begin{lemma}[\textbf{Lemma \ref{l2:sarahmain}}]\label{lemma6}
Suppose that Assumptions \ref{as1}, \ref{as2} hold. Let the stepsize $\gamma \leqslant \frac{1}{3L}$. Then for Algorithm \ref{alg:sarah} a valid estimate is
\begin{align*}
        \left\| \nabla f(x_s^0) - \frac{1}{n+1}\sum\limits_{t=0}^n v_s^t\right\|^2 & \leqslant 9\gamma^2L^2\|v_{s}\|^2 + 36\gamma^2L^2n^2\|v_{s-1}\|^2.
\end{align*}
\begin{proof}
        To begin with, in {Lemma \ref{ngl2}}, we obtain
        \begin{equation}
        \label{l6:ineq1}
            \left\| \nabla f(x_s^0) - \frac{1}{n+1} \sum\limits_{t=0}^n v_s^t\right\|^2 \leqslant 2\|\nabla f(x_s^0) - v_s \|^2 + \frac{2L^2}{n+1}\sum\limits_{t=1}^n \|x_s^t - x_s^{t-1}\|^2.
        \end{equation}
        Let us show what $v_s$ is (here we use Line  \ref{algsarah:line8} of Algorithm \ref{alg:sarah}):
        \begin{align}
            \notag v_s & = \widetilde{v}_{s-1}^{n+1} = \frac{n-1}{n} \widetilde{v}_{s-1}^{n} + \frac{1}{n}\nabla f_{\pi_{s-1}^{n}} (x_{s-1}^{n}) \\
            \notag & = \frac{n-1}{n}\cdot\frac{n-2}{n-1} \widetilde{v}_{s-1}^{n-1} + \frac{n-1}{n}\cdot\frac{1}{n-1}\nabla f_{\pi_{s-1}^{n-1}}(x_{s-1}^{n-1})
            + \frac{1}{n}\nabla f_{\pi_{s-1}^{n}} (x_{s-1}^{n})\\
            \notag& = \frac{n-1}{n}\cdot\frac{n-2}{n-1}\cdot\ldots\cdot 0\cdot \widetilde{v}_{s-1}^1 + \frac{1}{n}\sum\limits_{t = 1}^n\nabla f_{\pi_{s-1}^t}(x_{s-1}^t)\\
            \label{l6:ineq2}& \overset{(i)}{=} \frac{1}{n}\sum\limits_{t = 1}^n\nabla f_{\pi_{s-1}^t}(x_{s-1}^t),
        \end{align}
        where equation (\textit{i}) is correct due to initialization $\widetilde{v}_{s-1}^1 = 0$ (Line \ref{tvs} of Algorithm \ref{alg:sarah}). In that way, using \eqref{l6:ineq1} and \eqref{l6:ineq2},
        \begin{eqnarray*}
            \left\| \nabla f(x_s^0) - \frac{1}{n+1}\sum\limits_{t=0}^n v_s^t\right\|^2 &\leqslant& 2\left\|\nabla f(x_s^0) - \frac{1}{n}\sum\limits_{t = 1}^n\nabla f_{\pi_{s-1}^t}(x_{s-1}^t) \right\|^2 \\
            & & + \frac{2L^2}{n+1}\sum\limits_{t=1}^n \|x_s^t - x_s^{t-1}\|^2.
            \end{eqnarray*}
            Then, using \eqref{eq:finite-sum},
            \begin{eqnarray}
             \notag \left\| \nabla f(x_s^0) - \frac{1}{n+1}\sum\limits_{t=0}^n v_s^t\right\|^2 &\leqslant& 2\left\|\frac{1}{n}\sum\limits_{t = 1}^n\left[\nabla f_{\pi_{s-1}^t}(x_s^0) - \nabla f_{\pi_{s-1}^t}(x_{s-1}^t) \right]\right\|^2 \\
             \notag& &+  \frac{2L^2}{n+1}\sum\limits_{t=1}^n \|x_s^t - x_s^{t-1}\|^2 \\
            \notag &\overset{\eqref{ineq1}, \text{Ass. \ref{as1}}}{\leqslant}& \frac{2L^2}{n}\sum\limits_{t=1}^n\|x_{s-1}^t - x_s^0\|^2 + \frac{2L^2}{n+1}\sum\limits_{t=1}^n \|x_s^t -x_s^{t-1}\|^2 \\
            \notag&\overset{\eqref{ineq:square}}{\leqslant}& \frac{4L^2}{n}\sum\limits_{t=1}^n\|x_{s-1}^t - x_{s-1}^0\|^2 + \frac{4L^2}{n}\sum\limits_{t=1}^n\|x_s^0 - x_{s-1}^0\|^2 \\
            \label{l6:ineq3}& & + \frac{2L^2}{n+1}\sum\limits_{t=1}^n \|x_s^t - x_s^{t-1}\|^2.
        \end{eqnarray}
        Now we have to bound these three terms. Let us begin with the $\sum\limits_{t=1}^n\|x_s^t - x_s^{t-1}\|^2$ norm.
        \begin{align}
        \label{l6:ineq4}
            \sum\limits_{t=1}^n\|x_s^t - x_s^{t-1}\|^2 = \gamma^2\sum\limits_{t=1}^n\|v_s^{t-1}\|^2 = \gamma^2\sum\limits_{t=0}^{n-1}\|v_s^{t}\|^2.
        \end{align}
        Now we estimate $\|v_{s}^t\|^2$. For $t\geqslant 1$:
        \begin{eqnarray*}
            \|v_{s}^t\|^2 &=& \left\|v_{s}^{t-1} + \frac{1}{n}\left(\nabla f_{\pi_{s}^t}(x_{s}^t) - \nabla f_{\pi_{s}^t}(x_{s}^{t-1})\right) \right\|^2 
            \\ &\overset{\eqref{ineq:square}}{\leqslant}&  \left(1 + \frac{1}{\beta}\right)\|v_{s}^{t-1}\|^2 + \frac{(1 + \beta)L^2}{n^2}\|x_{s}^t - x_{s}^{t-1}\|^2  \\
            &\overset{\eqref{ineq:square}}{\leqslant}& \left(1 + \frac{1}{\beta}\right)^2\|v_{s}^{t-2}\|^2 + \frac{1}{n^2}\left(1 + \frac{1}{\beta}\right)(1 + \beta)L^2\|x_{s}^{t-1} - x_{s}^{t-2}\|^2 \\
            & & + \frac{1}{n^2}(1 + \beta)L^2\|x_{s}^t - x_{s}^{t-1}\|^2 \\
            &\overset{\eqref{ineq:square}}{\leqslant}& \left(1 + \frac{1}{\beta}\right)^t\|v_{s}\|^2 + \frac{1}{n^2}(1 + \beta)L^2 \sum\limits_{k = 1}^t \left(1 + \frac{1}{\beta}\right)^{k-1}\|x_{s}^{t-k+1} - x_{s}^{t-k}\|^2 \\
            &\underset{\beta = t}{\overset{\eqref{ineq:square}}{\leqslant}}& \left(1 + \frac{1}{t}\right)^t \|v_{s}\|^2 +\frac{1}{n^2} (1+t)\left(1 + \frac{1}{t}\right)^t L^2\sum\limits_{k = 1}^t\|x_{s}^k - x_{s}^{k-1}\|^2.
        \end{eqnarray*}
        Then, utilizing the property of the exponent $\left(\left(1 + \frac{1}{t}\right)^t \leqslant e\right)$ and $t\leqslant n-1$ \eqref{l6:ineq4}, we get an important inequality (for $0\leqslant t\leqslant n-1$, since for $t=0$ we have $\|v_s^t\|^2 = \|v_s\|^2$ and desired inequality becomes trivial):
        \begin{align}
        \label{l6:ineq5}
            \|v_{s}^t\|^2 & ~\leqslant e \|v_{s}\|^2 + \frac{e L^2}{n}\sum\limits_{k = 1}^t\|x_{s}^k - x_{s}^{k-1}\|^2.
        \end{align}
        Now we substitute \eqref{l6:ineq5} to \eqref{l6:ineq4} and obtain
        \begin{align*}
            \sum\limits_{t=1}^n\|x_s^t - x_s^{t-1}\|^2 & = \gamma^2\sum\limits_{t=0}^{n-1}\|v_s^{t}\|^2 \leqslant e\gamma^2 n \|v_s\|^2 + \frac{e\gamma^2 L^2}{n}\sum\limits_{t=0}^{n-1}\sum\limits_{k=1}^{t} \|x_{s}^k - x_{s}^{k-1}\|^2\\
            & \leqslant e\gamma^2 n \|v_s\|^2 + e\gamma^2 L^2\sum\limits_{t=1}^{n-1}\|x_{s}^t - x_{s}^{t-1}\|^2\\
            & \leqslant e\gamma^2 n \|v_s\|^2 + e\gamma^2 L^2\sum\limits_{t=1}^{n}\|x_{s}^t - x_{s}^{t-1}\|^2.
        \end{align*}
        Straightforwardly expressing $\sum\limits_{t=1}^n\|x_s^t - x_s^{t-1}\|^2$ we obtain desired estimation:
        \begin{align*}
            \sum\limits_{t=1}^n\|x_s^t - x_s^{t-1}\|^2 \leqslant \frac{e\gamma^2 n \|v_s\|^2}{1 - e\gamma^2 L^2} \overset{e<3}{\leqslant}\frac{3\gamma^2 n \|v_s\|^2}{1 - 3\gamma^2L^2}.
        \end{align*}
        To finish this part of proof it remains for us to choose appropriate $\gamma$. In Lemma \ref{lemma4} we require $\gamma \leqslant\frac{1}{L(n+1)}$. There we are satisfied with even large values of $\gamma$. Let us estimate obtained expression with $\gamma\leqslant\frac{1}{L(n+1)}\leqslant\frac{1}{3L}$. Now we provide final estimation of this norm:
        \begin{align}
        \label{l6:ineq6}
            \sum\limits_{t=1}^n\|x_s^t - x_s^{t-1}\|^2 \leqslant \frac{9}{2}\gamma^2 n\|v_s\|^2.
        \end{align}
        Let us proceed our estimation of \eqref{l6:ineq3} with the $\sum\limits_{t=1}^n \|x_{s-1}^t - x_{s-1}^0\|^2$ term.
        \begin{equation}
            \label{l6:ineq7}
            \sum\limits_{t=1}^n \|x_{s-1}^t - x_{s-1}^0\|^2 = \gamma^2\sum\limits_{t=1}^n \left\|\sum\limits_{k = 0}^{t-1} v_{s-1}^k\right\|^2 \overset{\eqref{ineq1}}{\leqslant} \gamma^2\sum\limits_{t=1}^n t \sum\limits_{k = 0}^{t-1} \|v_{s-1}^k\|^2 \leqslant \gamma^2 n^2 \sum\limits_{t = 0}^{n-1}  \|v_{s-1}^t\|^2.
        \end{equation}
        Note, that we have already estimated $\|v_s^t\|^2$ term for $0\leqslant t\leqslant n-1$ \eqref{l6:ineq5}. Furthermore, we can make the same estimate for the terms in the $(s-1)$-th epoch and write 
        \begin{align}
        \label{l6:ineq8}
            \|v_{s-1}^t\|^2 & ~\leqslant e \|v_{s-1}\|^2 + \frac{e L^2}{n}\sum\limits_{k = 1}^t\|x_{s-1}^k - x_{s-1}^{k-1}\|^2.
        \end{align}
        Now we substitute \eqref{l6:ineq8} to \eqref{l6:ineq7} to obtain
        \begin{align}
            \notag\sum\limits_{t=1}^n \|x_{s-1}^t - x_{s-1}^0\|^2 &\leqslant \gamma^2 n^2\sum\limits_{t=0}^{n-1} \left(e\|v_{s-1}\|^2 + \frac{e L^2}{n}\sum\limits_{k = 1}^t\|x_{s-1}^k - x_{s-1}^{k-1}\|^2\right)\\
            \notag& \leqslant \gamma^2 n^3 e \|v_{s-1}\|^2 + e\gamma^2 L^2n\sum\limits_{t = 0}^{n-1}\sum\limits_{k = 1}^t\|x_{s-1}^k - x_{s-1}^{k-1}\|^2\\
            \notag& \leqslant \gamma^2 n^3 e\|v_{s-1}\|^2 + e\gamma^2L^2 n^2 \sum\limits_{t = 1}^{n-1}\|x_{s-1}^t - x_{s-1}^{t-1}\|^2\\
            \label{l6:ineq9}
            & \leqslant \gamma^2 n^3 e\|v_{s-1}\|^2 + e\gamma^2L^2 n^2 \sum\limits_{t = 1}^{n}\|x_{s-1}^t - x_{s-1}^{t-1}\|^2.
        \end{align}
        Note, that we have already estimated the $\sum\limits_{t = 1}^{n}\|x_{s}^t - x_{s}^{t-1}\|^2$ term \eqref{l6:ineq6}. Furthermore, we can make the same estimate for the term in the $(s-1)$-th epoch and write 
        \begin{align}
        \label{l6:ineq10}
            \sum\limits_{t=1}^n\|x_{s-1}^t - x_{s-1}^{t-1}\|^2 \leqslant \frac{9}{2}\gamma^2 n\|v_{s-1}\|^2.
        \end{align}
        Substituting \eqref{l6:ineq10} to \eqref{l6:ineq9} we derive
        \begin{align*}
            \sum\limits_{t=1}^n \|x_{s-1}^t - x_{s-1}^0\|^2 \leqslant 3\gamma^2 n^3 \|v_{s-1}\|^2 + \frac{27}{2}\gamma^4L^2 n^3\|v_{s-1}\|^2.
        \end{align*}
        Using our $\gamma\leqslant\frac{1}{3L}$ choice,
        \begin{align}
        \label{l6:ineq11}
            \sum\limits_{t=1}^n \|x_{s-1}^t - x_{s-1}^0\|^2 \leqslant 3\gamma^2 n^3 \|v_{s-1}\|^2 + \frac{3}{2}\gamma^2 n^3\|v_{s-1}\|^2 = \frac{9}{2}\gamma^2 n^3\|v_{s-1}\|^2.
        \end{align}
        It remains for us to estimate the $\sum\limits_{t=1}^n\|x_s^0 - x_{s-1}^0\|^2$ term. The estimate is quite similar to the previous one:
        \begin{align*}            
            \sum\limits_{t=1}^n \|x_s^0 - x_{s-1}^0\|^2  & = \gamma^2\sum\limits_{t=1}^n \left\|\sum\limits_{k = 0}^{n-1} v_{s-1}^{k-1}\right\|^2 \overset{\eqref{ineq:square}}{\leqslant} \gamma^2\sum\limits_{t=1}^n n \sum\limits_{k = 0}^{n-1} \|v_{s-1}^k\|^2 \leqslant \gamma^2 n^2 \sum\limits_{t = 0}^{n-1}  \|v_{s-1}^t\|^2. 
        \end{align*}
        We obtain the estimate as in \eqref{l6:ineq7}. Thus, proceed similarly as we did for the previous term, we obtain
        \begin{align}
        \label{l6:ineq12}
            \sum\limits_{t=1}^n \|x_{s}^0 - x_{s-1}^0\|^2 \leqslant \frac{9}{2}\gamma^2 n^3\|v_{s-1}\|^2.
        \end{align}
        Now we can apply the upper bounds obtained in \eqref{l6:ineq6}, \eqref{l6:ineq11}, \eqref{l6:ineq12} to \eqref{l6:ineq3} and have
        \begin{align*}
            \left\| \nabla f(\omega_s) - \frac{1}{n+1} \sum\limits_{i=0}^n v_s^i\right\|^2 & \leqslant 18\gamma^2L^2n^2\|v_{s-1}\|^2 + 18\gamma^2L^2n^2\|v_{s-1}\|^2 + 9\gamma^2L^2\|v_s\|^2 \\
            & = 9\gamma^2L^2\|v_{s}\|^2 + 36\gamma^2L^2n^2\|v_{s-1}\|^2,
        \end{align*}
        which ends the proof.
\end{proof}
\end{lemma}
        
\begin{theorem}[\textbf{Theorem \ref{th1:sarahmain}}]\label{theorem3}
Suppose Assumptions \ref{as1}, \ref{as2nonconvex} hold. Then Algorithm \ref{alg:sarah} with $\gamma\leqslant\frac{1}{20L(n+1)}$ to reach $\varepsilon$-accuracy, where $\varepsilon^2 = \frac{1}{S}\sum\limits_{s=1}^{S} \|\nabla f(x_s^0)\|^2$, needs
   %\vspace{-3mm}
    \begin{equation*}
        \mathcal{O} \left(\frac{nL}{\varepsilon^2}\right) ~~ \text{iterations and oracle calls.}
    \end{equation*}
        \begin{proof}
        We combine the result of Lemma \ref{lemma4} with the result of Lemma \ref{lemma6} and obtain
        \begin{align*}
            f(x_{s+1}^0) &\leqslant f(x_s^0) - \frac{\gamma(n+1)}{2}\|\nabla f(x_s^0)\|^2 \\
            & \quad + \frac{\gamma(n+1)}{2}\left(9\gamma^2L^2\|v_{s}\|^2 + 36\gamma^2L^2n^2\|v_{s-1}\|^2\right).
        \end{align*}
        We subtract $f(x^*)$ from both parts:
        \begin{align*}
            f(x_{s+1}^0) - f(x^*) &\leqslant f(x_s^0) - f(x^*) - \frac{\gamma(n+1)}{2}\|\nabla f(x_s^0)\|^2 \\
            & \quad + \frac{\gamma(n+1)}{2}\left(9\gamma^2L^2\|v_{s}\|^2 + 36\gamma^2L^2n^2\|v_{s-1}\|^2\right)\\
            & = f(x_s^0) - f(x^*) - \frac{\gamma(n+1)}{4}\|\nabla f(x_s^0)\|^2 \\
            & \quad + \frac{\gamma(n+1)}{2}\left(9\gamma^2L^2\|v_{s}\|^2 + 36\gamma^2L^2n^2\|v_{s-1}\|^2\right) \\
            & \quad - \frac{\gamma(n+1)}{4}\|\nabla f(x_s^0)\|^2.
            \end{align*}
       Then, transforming the last term by using \eqref{ineq:square} with $\beta=1$, we get
        \begin{align*}
            f(x_{s+1}^0) - f(x^*) &\leqslant f(x_s^0) - f(x^*) - \frac{\gamma(n+1)}{4}\|\nabla f(x_s^0)\|^2 \\
           & \quad  + \frac{\gamma(n+1)}{2}\left(9\gamma^2L^2\|v_{s}\|^2 + 36\gamma^2L^2n^2\|v_{s-1}\|^2\right) \\
            & \quad- \frac{\gamma(n+1)}{8}\|v_s\|^2 + \frac{\gamma(n+1)}{4}\|v_s - \nabla f(x_s^0)\|^2.
            \end{align*}
        Using Lemma \ref{lemma6} to $\|v_s - \nabla f(x_s^0)\|^2$ (specially $\frac{4L^2}{n}\cdot\eqref{l6:ineq11} + \frac{4L^2}{n}\cdot\eqref{l6:ineq12}$),
        \begin{align*}
            f(x_{s+1}^0) - f(x^*) &\leqslant f(x_s^0) - f(x^*) - \frac{\gamma(n+1)}{4}\|\nabla f(x_s^0)\|^2 \\
            & \quad + \frac{\gamma(n+1)}{2}\left(9\gamma^2L^2\|v_{s}\|^2 + 36\gamma^2L^2n^2\|v_{s-1}\|^2\right) \\
            & \quad- \frac{\gamma(n+1)}{8}\|v_s\|^2 + \frac{\gamma(n+1)}{4}\cdot 36\gamma^2L^2n^2\|v_{s-1}\|^2.
        \end{align*}
        Combining alike expressions,
        \begin{align}
            \notag f(x_{s+1}^0) - f(x^*) + \frac{\gamma(n+1)}{4}\|\nabla f(x_s^0)\|^2 &\leqslant f(x_s^0) - f(x^*) - \frac{\gamma(n+1)}{8}\left(1 - 36 \gamma^2L^2\right)\|v_s\|^2 \\ 
            \label{t3:ineq1}& \quad + \gamma(n+1)\cdot 27\gamma^2L^2n^2\|v_{s-1}\|^2. 
        \end{align}
        Using $\gamma \leqslant \frac{1}{20L(n+1)}$ (note it is the smallest stepsize from all the steps we used before, so all previous transitions are correct), we get
        \begin{align*}
            f(x_{s+1}^0) - f(x^*) &+ \frac{1}{10}\gamma(n+1)\|v_s\|^2 + \frac{\gamma(n+1)}{4}\|\nabla f(\omega_s)\|^2\\
            &\leqslant f(x_s^0) - f(x^*) + \frac{1}{10}\gamma(n+1)\|v_{s-1}\|^2.
        \end{align*}
        Next, denoting $\Delta_{s} = f(x_{s+1}^0) - f(x^*) + \frac{1}{10}\gamma(n+1)\|v_{s}\|^2$, we obtain
        \begin{equation*}
            \frac{1}{S}\sum\limits_{s=1}^{S} \|\nabla f(x_s^0)\|^2 \leqslant \frac{4\left[\Delta_0 - \Delta_{S}\right]}{\gamma(n+1)S}.
        \end{equation*}    
        We choose $\varepsilon^2 = \frac{1}{S}\sum\limits_{s=1}^{S} \|\nabla f(x_s^0)\|^2$ as criteria. Hence, to reach $\varepsilon$-accuracy we need $\mathcal{O}\left(\frac{L}{\varepsilon^2}\right)$ epochs and $\mathcal{O}\left(\frac{nL}{\varepsilon^2}\right)$ iterations. Additionally, we note that the oracle complexity of our algorithm is also equal to $\mathcal{O}(\frac{nL}{\varepsilon^2})$, since at each iteration the algorithm computes the stochastic gradient at only two points. This ends the proof.
\end{proof}
\end{theorem}

%%%%%%%%%%%%%%%%%%%%%%%%%
\subsection{Strongly convex setting} 
\begin{theorem}[\textbf{Theorem \ref{th2:sarahmain}}]\label{theorem4}
Suppose Assumptions \ref{as1}, \ref{as2stronglyconvex} hold. Then Algorithm \ref{alg:sarah} with $\gamma\leqslant\frac{1}{20L(n+1)}$ to reach $\varepsilon$-accuracy, where $\varepsilon = f(x_{S+1}^0)-f(x^*)$, needs
   %\vspace{-3mm}
    \begin{equation*}
        \mathcal{O} \left(\frac{nL}{\mu}\log \frac{1}{\varepsilon}\right)~~ \text{iterations and oracle calls.}
    \end{equation*}
        
\begin{proof}
Under Assumption \ref{as2stronglyconvex}, which states that the function is strongly convex, the \eqref{PL} condition is automatically satisfied. Therefore, 
\begin{align*}
    f(x_{s+1}^0) - f(x^*) + \frac{\gamma\mu(n+1)}{2}\left( f(x_s^0) - f(x^*)\right) &\leqslant  f(x_{s+1}^0) - f(x^*) + \frac{\gamma(n+1)}{4}\|\nabla f(x_s^0)\|^2.
\end{align*}
Thus, using \eqref{t3:ineq1},
\begin{align*}
f(x_{s+1}^0) - f(x^*) &+ \frac{\gamma\mu(n+1)}{2}\left( f(x_s^0) - f(x^*)\right) \leqslant f(x_s^0) - f(x^*) \\
& - \frac{\gamma(n+1)}{8}\left(1 - 36\gamma^2L^2\right)\|v_s\|^2 + \gamma(n+1)\cdot 27\gamma^2L^2n^2\|v_{s-1}\|^2.
\end{align*}
Using $\gamma \leqslant \frac{1}{20L(n+1)}$ and assuming $n \geqslant 2$, we get
\begin{align*}
    f(x_{s+1}^0) - f(x^*) + \frac{1}{10}\gamma(n+1)\|v_s\|^2
    &\leqslant \left(1-\frac{\gamma\mu(n+1)}{2}\right)\left(f(\omega_s) - f(x^*)\right) \\
    & \quad + \frac{1}{10}\gamma(n+1)\cdot \left(1-\frac{\gamma\mu(n+1)}{2}\right)\|v_{s-1}\|^2.
\end{align*}
Next, denoting $\Delta_{s} = f(x_{s+1}^0) - f(x^*) + \frac{1}{10}\gamma(n+1)\|v_{s}\|^2$, we obtain the final convergence over one epoch:
\begin{align*}
    \Delta_{s+1} \leqslant \left(1-\frac{\gamma\mu(n+1)}{2}\right) \Delta_s.
\end{align*}
Going into recursion over all epoch,
\begin{align*}
    f(x_{S+1}^0) - f(x^*)\leqslant\Delta_{S} \leqslant \left(1-\frac{\gamma\mu(n+1)}{2}\right)^{S+1}\Delta_0.
\end{align*}
We choose $\varepsilon = f(x_{S+1}^0)-f(x^*)$ as criteria. Then to reach $\varepsilon$-accuracy we need $\mathcal{O}\left(\frac{L}{\mu}\log\left(\frac{1}{\varepsilon}\right)\right)$ epochs and $\mathcal{O}\left(\frac{nL}{\mu}\log\left(\frac{1}{\varepsilon}\right)\right)$ iterations. Additionally, we note that the oracle complexity of our algorithm is also equal to $\mathcal{O}\left(\frac{nL}{\mu}\log\left(\frac{1}{\varepsilon}\right)\right)$, since at each iteration the algorithm computes the stochastic gradient at only two points.
\end{proof} 
\end{theorem}

\section{LOWER BOUNDS} \label{sec:lower_bound}

In this section we provide the proof of the lower bound on the amount of oracle calls in the class of the first-order algorithms with shuffling heuristic that find the solution of the non-convex objective finite-sum function. We follow the classical way by presenting the example of function and showing the minimal number of oracles needs to solve the problem. We consider the following function:
\begin{align*}
    l(x) = -\Psi(1)\Phi([x]_1) + \sum\limits_{j=2}^d \left(\Psi(-[x]_{j-1})\Phi(-[x]_j) - \Psi([x]_{j-1})\Phi([x]_j)\right),
\end{align*}
where $[x]_j$ is the $j$-th coordinate of the vector $x\in\mathbb R ^d$,
\begin{align*}
    \Psi(z) = \begin{cases}
        0, &\text{~if~} z\leqslant\frac{1}{2}\\
        \exp\left(1 - \frac{1}{(2z-1)^2}\right), &\text{~if~} z > \frac{1}{2}
    \end{cases}
\end{align*}
and
\begin{align*}
    \Phi(z) = \sqrt{e}\int\limits_{-\infty}^z \exp\left(-\frac{t^2}{2}\right) dt.
\end{align*}
We also define the following function:
\begin{align*}
    \text{prog}(x) = \begin{cases}
        0, &\text{~if~} x=0\\
        \underset{1\leqslant j \leqslant d}{\max} \left\{j : [x]_j \neq 0\right\}, &\text{~otherwise~}
    \end{cases},
\end{align*}
where $x\in\mathbb R ^d$. In the work \citep{arjevani2023lower} it was shown, that function $l(x)$ satisfies the following properties:
\begin{align*}
    &\forall x\in\mathbb R^d ~~ l(0) - \underset{x}{\inf}~ l(x) \leqslant \Delta_0 d,\\
    &l(x) \text{~is~} L_0\text{-smooth with~} L_0=152,\\
    &\forall x\in\mathbb R^d ~~ \|l(x)\|_{\infty}\leqslant G_0 \text{~with~} G_0=23,\\
    &\forall x\in\mathbb R^d : [x]_d = 0 ~~ \|l(x)\|_{\infty} \geqslant 1,\\
    &l(x) \text{~is the zero-chain function, i.e.,~} \text{prog}(\nabla l(x)) \leqslant \text{prog}(x) + 1.
\end{align*}

\begin{lemma}[Lemma C.9 from \citep{metelev2024decentralized}]\label{lemma_L_0-smooth}
    Each $l_j(x)$ is $L_0$-smooth, where
    \begin{align*}
        l_j(x) = \begin{cases}
                    -\Psi(1)\Phi([x]_1), &\text{~if~} j=1\\
                    \Psi(-[x]_{j-1})\Phi(-[x]_j) - \Psi([x]_{j-1})\Phi([x]_j), &\text{~otherwise}
                \end{cases}.
    \end{align*}
\end{lemma}

\subsection{Proof of Theorem \ref{thm:lower}}
\begin{theorem}[\textbf{Theorem \ref{thm:lower}}]
    For any $L > 0$ there exists a problem \eqref{eq:finite-sum} which satisfies Assumption \ref{as1}, such that for any output of first-order algorithm, number of oracle calls $N_c$ required to reach $\varepsilon$-accuracy is lower bounded as
    \begin{align*}
        N_c = \Omega\left(\frac{L\Delta}{\varepsilon^2}\right).
    \end{align*}
\end{theorem}
        \begin{proof}
            To begin with, we need to decompose the function $l(x)$ to the finite-sum:
            \begin{align*}
                l(x) = \sum\limits_{j=1}^d l_j(x),
            \end{align*}
            where index $j$ responds the definition of $l(x)$, i.e.,
            \begin{align*}
                l_j(x) = \begin{cases}
                    -\Psi(1)\Phi([x]_1), &\text{~if~} j=1\\
                    \Psi(-[x]_{j-1})\Phi(-[x]_j) - \Psi([x]_{j-1})\Phi([x]_j), &\text{~otherwise}
                \end{cases}.
            \end{align*}
            Now we design the following objective function:
            \begin{align*}
                f(x) = \frac{1}{n}\sum\limits_{i=1}^n f_i(x),
            \end{align*}
            where $f_i(x) = \frac{LC^2}{L_0}\sum\limits_{j\equiv i \mod n} l_j\left(\frac{x}{C}\right)$.
            Since each $l_j\left(\cdot\right)$ is $L_0$-smooth (according to Lemma \ref{lemma_L_0-smooth}) and for $j\equiv i\mod n$ the gradients $\nabla l_j(x)$ are separable, than for any $x_1, x_2 \in\mathbb R^d$ it implies
            \begin{eqnarray*}
                \|\nabla f_i(x_1) - \nabla f_i(x_2)\|^2 &=& \frac{L^2 C^2}{L_0^2}\left\|\sum\limits_{j\equiv i \mod n} \left(\nabla l_j\left(\frac{x_1}{C}\right) - \nabla l_j\left(\frac{x_2}{C}\right)\right)\right\|^2 \\
                &\leqslant& \frac{L^2 C^2}{L_0^2}\frac{L_0^2}{C^2} \|x_1 - x_2\|^2 = L^2\|x_1 - x_2\|^2.
            \end{eqnarray*}
            It means, each function $f_i(x)$ is $L$-smooth. Moreover, since $f(x) = \frac{LC^2}{nL_0} l\left(\frac{x}{C}\right)$,
            \begin{eqnarray}
                \Delta = \notag f(0) - \underset{x}{\inf}~ f(x) &=& \frac{LC^2}{nL_0} \left(l(0) - \underset{x}{\inf}~ l\left(\frac{x}{C}\right)\right)\\
                \label{thlb:ineq1}&=& \frac{LC^2}{nL_0} \left(l(0) - \underset{x}{\inf}~ l(x)\right) \leqslant \frac{LC^2\Delta_0 d}{nL_0}.
            \end{eqnarray}
            Now we show, how many oracle calls we need to have progress in one coordinate fo vector $x$. At the current moment, we need a specific piece of function, because according to structure of $l(x)$, each gradient estimation can "defreeze" at most one component and only a computation on a certain block makes it possible. Formally, since $\frac{1}{n}\sum\limits_{i=1}^n f_i(x) = \frac{LC}{dL_0} l\left(\frac{x}{C}\right)$,
            \begin{align*}
                \text{prog}(\nabla f_i(x)) \begin{cases}
                    =\text{prog}(x) + 1, &\text{~if~} i = \text{prog}(x)\mod n\\
                    \leqslant \text{prog}(x), &\text{~otherwise}
                \end{cases}.
            \end{align*}
            Now, we need to show the probability of choosing the necessary piece of function, according to the shuffling heuristic. This probability at the first iteration of the epoch, i.e., iteration $t$, such that $t\mod n = 1$, is obviously $\frac{1}{n}$. At the second iteration of the epoch -- $\frac{n-1}{n}\cdot\frac{1}{n-1} = \frac{1}{n}$. Thus, at the $k$-th iteration of the epoch, the desired probably is $\frac{n-1}{n}\cdot\frac{n-2}{n-1}\cdot\ldots\cdot\frac{1}{n-k+1} = \frac{1}{n}$. In that way, the expected amount of gradient calculations though the epoch is
            \begin{align*}
                \sum\limits_{i=1}^n \frac{i}{n} = \frac{n+1}{2} \leqslant n.
            \end{align*}
            Since epochs is symmetrical in a sense of choosing indices, we need to perform $n$ oracle calls at each moment of training. Thus, after $T$ oracle calls, we can change only $\frac{T}{n}$ coordinate of vector $x$. Now, we can write the final estimate:
            \begin{eqnarray*}
                \mathbb E\|\nabla f(\hat{x})\|^2_2 &\geqslant& \mathbb E\|\nabla f(\hat{x})\|^2_\infty \geqslant \underset{[x]_d=0}{\min} \|\nabla f(\hat{x})\|^2_\infty = \frac{L^2C^2}{n^2L_0^2}\left\|\nabla l\left(\frac{\hat{x}}{C}\right)\right\|^2_\infty \geqslant \frac{L^2C^2}{n^2L_0^2}\\
                &\overset{\eqref{thlb:ineq1}}{\geqslant}& \frac{L\Delta}{n L_0 \Delta_0 d} = \frac{L\Delta}{L_0 \Delta_0 T}.
            \end{eqnarray*}
            Thus, lower bound on $T$ is $\Omega\left(\frac{L\Delta}{\varepsilon^2}\right)$.
        \end{proof}
\subsection{Proof of Theorem \ref{thm: quest}}
Before we start the proof, let us introduce other assumptions of smoothness for the complete analysis.
\begin{assumption}[Smoothness of each $f_i$]
\label{asm: each}
    Each function $f_i$ is $L_i$-smooth, i.e., it satisfies $$\|\nabla f_i(x) - \nabla f_i(y)\| \leq L_i\|x - y\|$$ for any $x, y \in \mathbb{R}^d$.
\end{assumption}
\begin{assumption}[Average smoothness of $f$]
\label{asm: average}
    Function $f$ is $\hat{L}$-average smooth, i.e., it satisfies $$\mathbb{E}_i\left[\|\nabla f_i(x) - \nabla f_i(y)\|^2\right] \leq \hat{L}^2\|x - y\|^2$$ for any $x, y \in \mathbb{R}^d$.
\end{assumption}
Here we also assume that $L_i$ with $i = 1, \ldots, n$ and $\hat{L}$ are \textit{effective}: it means that these constants cannot be reduced.

If $\{f_i\}_{i=1}^n$ satisfies Assumption \ref{as1}, it automatically leads to the satisfaction of Assumption \ref{asm: each}, since $L_i$ can be chosen as $L$. Nevertheless, the effective constant of smoothness for $f_i$ can be less than $L$. As a consequence, we obtain the next result.
\begin{lemma}
    \label{lem: ineq-L}
    Suppose that Assumption \ref{as1} holds. Then, the set $\{f_i\}_{i=1}^n$ satisfies Assumptions \ref{asm: each} and \ref{asm: average}. Moreover,
    \begin{align*}
        \hat{L} \leq L,
    \end{align*}
    where $\hat{L}$ and $L$ are chosen effectively.
\end{lemma}
\begin{proof}
    Let $L_i$ be the constant of smoothness of $f_i$. Therefore, $L_i \leq L$, and $L$ is defined as $\max_{i} L_i$. Moreover, $\hat{L}^2$ is defined as
    \begin{align*}
        \hat{L}^2 = \sum\limits_{i=1}^n w_i L_i^2,
    \end{align*}
    where $\{w_i\}_{i=1}^n$ is probabilities for the sampling of $f_i$, i.e. $w = (w_1, \ldots, w_n)$ formalizes the discrete distribution over indices $i$ (the most common case: $w_i = \frac{1}{n}$; nevertheless, we consider an unified option). As a result, we have
    \begin{align*}
        \hat{L}^2 = \sum\limits_{i=1}^n w_i L_i^2 \leq \sum\limits_{i=1}^n w_i L^2 = L^2.
    \end{align*}
    This concludes the proof.
\end{proof}
Now we are ready to proof the Theorem \ref{thm: quest}.
\begin{theorem}[\textbf{Theorem \ref{thm: quest}}]
    For any $L > 0$ there is \textbf{no} problem \eqref{eq:finite-sum} which satisfies Assumption \ref{as1}, such that for any output of first-order algorithm, number of oracle calls $N_c$ required to reach $\varepsilon$-accuracy is lower bounded with $p > \frac{1}{2}$:
    \begin{align*}
        N_c = \Omega\left(\frac{n^pL\Delta}{\varepsilon^2}\right).
    \end{align*}
    \end{theorem}
\begin{proof}
    Let us assume that we can find the problem \eqref{eq:finite-sum} which satisfies Assumption \ref{as1}, such that for any output of first-order algorithm, number of oracle calls $N_c$ required to reach $\varepsilon$-accuracy is lower bounded as
    \begin{align*}
        N_c = \Omega\left(\frac{n^pL\Delta}{\varepsilon^2}\right)
    \end{align*}
    with $p > \frac{1}{2}$. Applying Lemma \ref{lem: ineq-L}, one can obtain
    \begin{align*}
        N_c = \Omega\left(\frac{n^pL\Delta}{\varepsilon^2}\right) \geq \Omega\left(\frac{n^p\hat{L}\Delta}{\varepsilon^2}\right), 
    \end{align*}
    which contradict existing results of upper bound in terms of $n$ under Assumption \ref{asm: average} (e.g. \cite{fang2018spider}). This finishes the proof.
\end{proof}

\end{document}